\documentclass[10pt,twocolumn,letterpaper]{article}

\usepackage{cvpr}              
\definecolor{cvprblue}{rgb}{0.21,0.49,0.74}
\usepackage[pagebackref,breaklinks,colorlinks,allcolors=cvprblue]{hyperref}
\usepackage{multirow}
\usepackage{relsize}
\usepackage{amsthm}
\usepackage{caption}
\usepackage[accsupp]{axessibility} 
\newtheorem{definition}{Definition}

\usepackage{amsmath}
\newcommand{\std}[1]{\scriptsize{$\pm$ #1}}

\def\paperID{18982} 
\def\confName{CVPR}
\def\confYear{2026}

\title{Rethinking Dataset Distillation: Hard Truths About Soft Labels}

 \author{Priyam Dey \thanks{Equal Contribution.\\ $^\dag$ Work done as a Pre-doctoral researcher at Vision and AI Lab.\\ Correspondence to Priyam Dey $<$priyamdey@iisc.ac.in$>$, Aditya Sahdev $<$adityasahdev.97@gmail.com$>$.}~~$^{1}$ \quad Aditya Sahdev \footnotemark[1]~~$^{1}$~$^\dag$ \quad Sunny Bhati $^{2}$~$^\dag$ \\ \quad Konda Reddy Mopuri$^{3}$ \quad Venkatesh Babu Radhakrishnan$^{1}$ \vspace{0.3em} \\
{\normalsize $^1$Vision and AI Lab, Indian Institute of Science, Bangalore} \\ \quad
\normalsize $^2$ University of Maryland, Baltimore County \quad
{\normalsize $^3$IIT Hyderabad} \quad \\
}

\usepackage{enumitem}
\begin{document}
\maketitle
\begin{abstract}
Despite the perceived success of large-scale dataset distillation (DD) methods, recent evidence \cite{qin2024a} finds that simple random image baselines perform on-par with state-of-the-art DD methods like SRe2L \cite{yin2024squeezerecoverrelabeldataset} due to the use of soft labels during downstream model training. This is in contrast with the findings in coreset literature, where high-quality coresets consistently outperform random subsets in the hard-label (HL) setting. To understand this discrepancy, we perform a detailed scalability analysis to examine the role of data quality under different label regimes, ranging from abundant soft labels (termed as SL+KD regime) to fixed soft labels (SL) and hard labels (HL). Our analysis reveals that high-quality coresets fail to convincingly outperform the random baseline in both SL and SL+KD regimes. In the SL+KD setting, performance further approaches near-optimal levels relative to the full dataset, regardless of subset size or quality, for a given compute budget. This performance saturation calls into question the widespread practice of using soft labels for model evaluation, where unlike the HL setting, subset quality has negligible influence. A subsequent systematic evaluation of five large-scale and four small-scale DD methods in the HL setting reveals that only RDED \cite{sun2024diversityrealismdistilleddataset} reliably outperforms random baselines on ImageNet-1K, but can still lag behind strong coreset methods due to its over-reliance on easy sample patches. Based on this, we introduce CAD-Prune, a compute-aware pruning metric that efficiently identifies samples of optimal difficulty for a given compute budget, and use it to develop CA2D, a compute-aligned DD method, outperforming current DD methods on ImageNet-1K at various IPC settings. Together, our findings uncover many insights into current DD research and establish useful tools to advance data-efficient learning for both coresets and DD.
\end{abstract}    
\section{Introduction}
\label{sec:intro}

\begin{figure*}
    \centering
    \includegraphics[width=\textwidth]{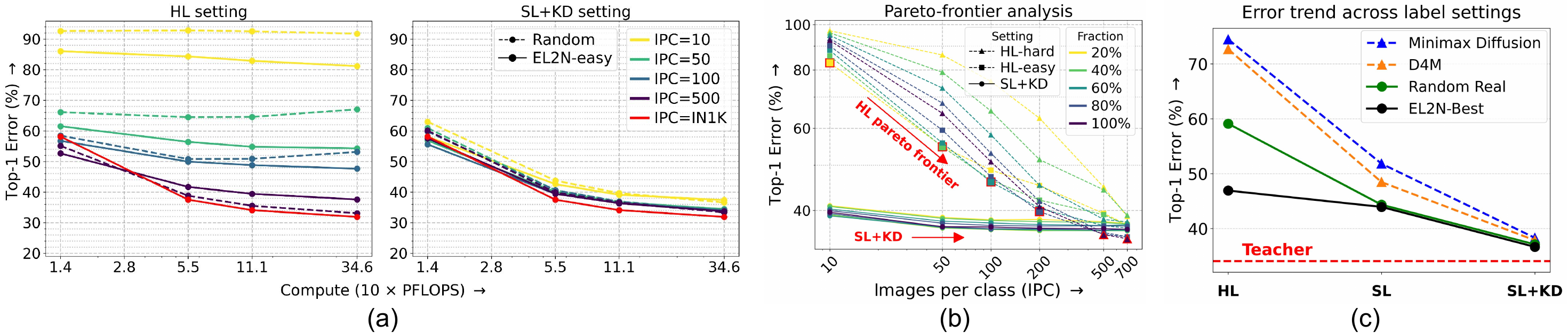}
    \caption{\textbf{Scalability analysis of various coresets and large-scale DD sets on ImageNet-1K in SL+KD regime.} \textbf{(Left)} Performance of coresets of varying quality (Random vs. EL2N-easy) and size (IPC 10–500+) across compute budgets equivalent to 2–50 epochs of full-dataset training. Unlike the HL setting, performance in SL+KD is dominated by compute, remains largely invariant to coreset quality and size, and quickly saturates to the near-optimal baseline of full dataset (red line of IPC=IN1K). \textbf{(Middle)} Pareto analysis \cite{sorscher2023neuralscalinglawsbeating} across data fractions shows \textit{no pareto-frontier} in the SL+KD regime, as all subsets achieve nearly identical accuracy across the IPC values. \textbf{(Right)} Error trends across label settings show that increasing the amount of soft-label supervision during student training drives all methods (coresets and DD) toward the teacher’s full-dataset performance regardless of their underlying quality.}
    \label{fig:motivation}
\end{figure*}

Recent success in dataset distillation (DD)\cite{wang2020datasetdistillation,sachdeva2023data} has been driven by methods \cite{yin2024squeezerecoverrelabeldataset,shao2024elucidating,du2024diversitydrivensynthesisenhancingdataset} that scale successfully to larger datasets like ImageNet-1K \cite{deng2009imagenet}. Much of this success stems from the widespread adoption of the SRe2L framework \cite{yin2024squeezerecoverrelabeldataset}, which decouples the costly synthesis step from the student model training. To achieve good downstream performance, the framework employs \textit{multiple soft labels} (one per augmented image) from a teacher to train the student model, a practice akin to knowledge distillation \cite{Hinton2015DistillingTK,beyer2022knowledge, gou2021knowledge} (we term this training setup as SL+KD). Recently, \citet{qin2024a} demonstrated that the primary factor behind the success of these methods stems from the use of these soft labels instead of the synthesized images, with randomly selected training images achieving performance comparable to then SOTA DD method SRe2L \cite{yin2024squeezerecoverrelabeldataset} with enough soft labels. Since then, several new state-of-the-art DD methods have been proposed including generative modeling-based approaches \cite{shao2024elucidating,sun2024diversityrealismdistilleddataset,su2024d4mdatasetdistillationdisentangled,gu2024efficient}, but they still struggle to convincingly beat the random subset baseline except in very low data budget regimes (see SL+KD panel in \cref{tab:main_in1k}).

This behavior stands in sharp contrast to results in coreset selection~\cite{guo2022deepcorecomprehensivelibrarycoreset,dyn_unc}, where SOTA approaches consistently outperform random baselines in the hard label (HL) setting by leveraging high-quality subsets, prompting a deeper investigation into the role of data quality in the presence of heavy soft labels. More broadly, we pose the following question: \textit{How does data quality, size and compute interplay in shaping performance behavior under the SL+KD regime?} We investigate this by analyzing the scaling behavior of high-quality coresets (such as EL2N~\cite{paul2021deep}) against random subsets in the SL+KD regime across different data and compute budgets, and find a stark difference in the results: Unlike the HL setting, where dataset size, quality and compute must be scaled together to achieve optimal performance \cite{kaplan2020scaling,goyal2024scaling}, 
performance in the SL+KD regime is \textit{primarily dictated by compute and remains largely unaffected by both dataset size and quality}.
Across diverse compute budgets, we find that on ImageNet-1K~\cite{deng2009imagenet}, using only 50–100 random training images per class (IPC) yields accuracy comparable to that of higher-quality or even bigger-sized coresets \textit{for the same compute budget} (see \cref{fig:motivation}(c) and the SL+KD plot in \cref{fig:motivation}(a)). More strikingly, \textit{for a fixed compute budget, all such subsets already reach near-optimal performance relative to the full dataset}, prompting the concerning question: \textit{Is there any scope of improvement to be made in dataset distillation by continuing to operate under the SL+KD regime?} 

Another common DD setting~\cite{guo2024losslessdatasetdistillationdifficultyaligned, cui2023scalingdatasetdistillationimagenet1k} that uses teacher supervision, but to a lesser extent (one soft label per image), is the \textit{fixed soft-label (SL) setting}. We extend our analysis of examining how quality, size, and compute shape performance to the SL setting as well, and observe a clear trend: while jointly scaling dataset size and compute remains important for performance, \textit{improvements in dataset quality offers only limited additional benefit} (see \cref{fig:sl_setting}(a)). A per-sample analysis using a modified EL2N score indicates that, unlike in HL setting, \textit{sample influence during the training phase in SL becomes highly uniform} (see \cref{fig:sl_setting}(b)), explaining the observed quality indifference in the scaling results. Therefore, given the limited role of data quality in both the soft-label regimes (SL and SL+KD), our study now emphasizes the HL setting, where data quality has been a primary driver of model performance \cite{sorscher2023neuralscalinglawsbeating}.

A systematic evaluation of five large-scale and four small-scale DD methods \textit{in the HL setting} finds that RDED \cite{sun2024diversityrealismdistilleddataset} is the only effective large-scale method, while TM and DATM achieve strong performance in the small-scale setting on TinyImageNet with ConvNet-D4, outperforming both random and coreset baselines across IPC budgets.
Given their strong performance, we evaluate whether methods such as TM and DATM can effectively scale to larger settings, like ResNet-18 on TinyImageNet. To do so, we turn our focus to their underlying distillation objectives and propose \textit{DCS (Distillation Correlation Score)}, an efficient metric that quickly assesses how well a given distillation objective aligns with downstream model generalization. It measures the correlation between the distillation loss and subsequently trained model's generalization error across a diverse collection of precomputed subsets that vary in quality and IPC, effectively bypassing the need for both dataset synthesis and downstream model training. Using DCS, we find that TM-based loss objectives \textit{do not scale effectively to larger models} like ResNet-18, where their loss objective values remain constant across various subsets despite varying performances of the underlying subsets used to calculate the loss objective (see \cref{fig:dcs_tm}). 

Given the lack of scalability of the trajectory-matching based methods, we turn our focus to the only performing method (RDED), which although outperform the random subset baseline across various IPC and compute budgets (see HL panel in \cref{tab:main_in1k}), 
it still doesn't surpass
stronger coreset methods like EL2N-Best \cite{paul2021deep,lee2024selmatcheffectivelyscalingdataset} consistently, potentially due to its simple underlying heuristic of selecting only most-confident sample patches from a larger set of images.
As noted in prior works of both coresets \cite{sorscher2023neuralscalinglawsbeating} and DD \cite{lee2024selmatcheffectivelyscalingdataset}, effective data subsets at any IPC budget should capture a diverse range of features with suitable difficulty levels aligned with the available compute (see optimal hardness plot of HL in \cref{fig:sl_setting}(c)). 
We therefore introduce an efficient and compute-aware pruning metric named \textit{CAD-Prune (Compute-Aware Difficulty Pruning)} designed to efficiently identify samples of appropriate difficulty for a given compute budget, and leverage this metric to develop an improved DD method called \textit{CA2D (Compute-Aware Dataset Distillation)}, which selects confident patches from examples of \textit{suitable difficulty at a given compute budget} to build an RDED-style synthetic set.
We show that CA2D effectively outperforms both RDED and the best coreset across various IPC settings on ImageNet-1K in the HL setting.
We summarize our contributions below:

\begin{itemize}[noitemsep, topsep=0pt, leftmargin=10pt]
    \item We conduct a systematic analysis of how dataset quality, size, and compute jointly affect performance across key label regimes in data-efficient learning. We show that: (a) In the SL+KD regime, \textit{performance is largely insensitive to dataset quality or size} and is instead governed by compute, \textit{approaching the near-optimal full-dataset accuracy} across different budgets; (b) in the fixed soft-label (SL) regime, dataset size still matters, but \textit{subset quality does not hold influence beyond very small IPC values}. We trace this to the tight clustering of difficulty scores during SL training, which effectively homogenizes sample quality; (c) in the HL setting, large-scale (RDED) and small-scale (TM, DATM) DD methods do outperform random baselines, with mixed results w.r.t best coresets.

    \item To analyze scalability of small-scale DD methods like TM and DATM, we propose \textbf{DCS (Distillation Correlation Score)}, a fast and effective metric that evaluates how well a distillation loss objective aligns with downstream generalization without requiring the expensive synthesis and training process. Using DCS, we find that trajectory-matching based synthesis objectives fail to scale to larger settings like ResNet-18 on TinyImageNet.
    \item We finally introduce \textbf{CAD-Prune}, a compute-aware pruning metric that generalizes reliably across various IPC and compute budgets, and use it to propose \textbf{CA2D}, a compute-aligned DD method built on the principle of compute-efficiency. We demonstrate state-of-the-art results on ImageNet-1K in HL setting, outperforming existing baselines across IPC settings.

\end{itemize}

\section{Preliminaries and Setup}
\label{sec:prelim}
We begin by introducing the notation used throughout this work. In a standard ML setup, we have a training dataset $\mathcal{D}_{train} = \{(x_i, y_i)\}_{i=1}^{|\mathcal{D}_{train}|}$, where 
each $(x_i, y_i)$ denote an image-label pair. A model parametrized by \(\theta\) is trained on \(\mathcal{D}_{train}\) by minimizing the empirical training loss $\ell_{\mathcal{D}_{train}}(\theta) = \mathbb{E}_{(x, y) \sim \mathcal{D}_{train}} (\ell(x, y, \theta))$. Thus, the optimal parameters are given by 
$\theta_{\mathcal{D}_{train}}^{*} = \arg\min_{\theta} \ell_{\mathcal{D}_{train}}(\theta)$. 

Building on this setup, we now consider the task of dataset distillation. 
Given a large training dataset $\mathcal{D}_{train}$, the goal is to \textit{synthesize} a compact set $\mathcal{S}$ that retains the essential information of $\mathcal{D}_{train}$ useful for various downstream tasks. We denote such an synthesized set as $\mathcal{S}^*$. When a model $\theta$ is trained on $\mathcal{S}^*$, it should achieve a minimal test loss $\ell_{\mathcal{D}_{test}}(\theta_{\mathcal{S}}^*)$, where $\ell_{\mathcal{D}_{test}}$ denotes loss obtained on a downstream testset by the model $\theta_{\mathcal{S}}^*$. That is, $\displaystyle \mathcal{S}^* = \arg\min_{\mathcal{S}} \ell_{\mathcal{D}_{test}}(\theta_{\mathcal{S}})$.
To obtain $\mathcal{S}^*$, different methods employ a surrogate distillation loss objective $\mathcal{L}_{\text{distill}}(\mathcal{S,D})$ which guides the optimization of the synthetic dataset $\mathcal{S}$ using a training algorithm $\phi$ on $\mathcal{D}$. Doing so implies an important assumption underlying this formulation: \textit{Minimization of} $\mathcal{L}_{\text{distill}}$ \textit{should yield a synthetic set} $\mathcal{\hat{S}}^*$ \textit{that, when used for training a model, should lead to good downstream generalization on a test set} $\mathcal{D}_{test}$. Therefore, every DD method aims to closely approximate $\mathcal{S}^*$ with $\mathcal{\hat{S}}^*$ obtained from their training method.
Given this background, we now dive deep into prevalent label regimes to understand the behavior of different data-efficient methods (DD / coresets) when operating under these regimes.  

\section{Label regimes of Data-Efficient Learning}
\label{sec:baselines}


\begin{table*}
\centering
\caption{\textbf{Performance comparison of large-scale DD methods with coreset selection methods on ImageNet-1K}. We compare hard label (HL), fixed soft label (SL) and cutmix-augmented soft labels (SL+KD) setting. Model architecture is ResNet-18. Best numbers are \textbf{bolded} within each method type (DD, coresets). Full dataset numbers are reported with the same compute used for the IPC setting. The substantial performance gap between methods (DD or coresets) closes when trained in SL+KD setting.}
\vspace{-0.2cm}
\setlength{\tabcolsep}{4pt}
\resizebox{\linewidth}{!}{
\label{tab:main_in1k}
\begin{tabular}{l|ccc|ccc|ccc}
\toprule
\multicolumn{1}{c|}{\multirow{2}{*}{\textbf{Method}}} & \multicolumn{3}{c|}{\textbf{Hard Label (HL) Setting}}                                       & \multicolumn{3}{c|}{\textbf{Fixed Soft Label (SL) Setting}}                                       & \multicolumn{3}{c}{\textbf{SL+KD Setting}}                                     \\
\cmidrule{2-10}
\multicolumn{1}{c|}{}                                 & \textbf{IPC 10}       & \textbf{IPC 50}       & \textbf{IPC 100}      & \textbf{IPC 10}       & \textbf{IPC 50}       & \textbf{IPC 100}      & \textbf{IPC 10}       & \textbf{IPC 50}       & \textbf{IPC 100}      \\
\midrule
\multicolumn{10}{c}{\textbf{Dataset Distillation (224 $\times$ 224)}} \\
\midrule
SRe2L~\cite{yin2024squeezerecoverrelabeldataset}                                                & 2.82 \std{0.11}           & 9.79 \std{0.08}           & 14.89 \std{0.20}          & 14.14 \std{0.19}          & 27.64 \std{0.09}          & 37.33 \std{0.07}          & 30.94 \std{0.24}          & 52.58 \std{0.03}          & 61.38 \std{0.15}          \\
DWA~\cite{du2024diversitydrivensynthesisenhancingdataset}                                                 & 2.61 \std{0.15}           & 8.26 \std{0.07}           & 10.22 \std{0.04}          & 14.21 \std{0.14}          & 28.42 \std{0.33}          & 34.73 \std{0.22}          & 32.72 \std{0.22}          & 53.25 \std{0.10}          & 57.74 \std{0.10}          \\
D4M~\cite{su2024d4mdatasetdistillationdisentangled}                                                  & 5.68 \std{0.21}           & 21.56 \std{0.17}          & 27.33 \std{0.23}          & 18.79 \std{0.16}          & \textbf{44.33 \std{0.13}}          & \textbf{51.53 \std{0.01}}          & \textbf{34.16 \std{0.54}}          & \textbf{58.08 \std{0.02}}          & \textbf{62.13 \std{0.01}}          \\
Minimax Diffusion~\cite{gu2024efficient}                                   & 6.32 \std{0.18}           & 19.62 \std{1.28}          & 25.54 \std{0.18}          & 17.97 \std{0.40}          & 40.30 \std{0.09}          & 48.18 \std{0.17}          & 33.88 \std{0.39}          & 57.77 \std{0.13}          & 61.68 \std{0.07}          \\
RDED~\cite{sun2024diversityrealismdistilleddataset}                                                 & \textbf{14.34 \std{0.17}} & \textbf{38.49 \std{0.41}}          & \textbf{44.36 \std{0.07}}          & \textbf{19.27 \std{0.46}}          & 38.78 \std{0.21}          & 44.91 \std{0.16}          & 32.10 \std{0.34}          & 54.42 \std{0.13}          & 58.61 \std{0.07}          \\
\midrule
\multicolumn{10}{c}{\textbf{Coreset Selection (224 $\times$ 224)}} \\
\midrule
Random Real                                          & 5.10 \std{0.04}           & 28.52 \std{0.23}          & 40.89 \std{0.04}          & 19.96 \std{0.08}          & 48.49 \std{0.14}          & 55.65 \std{0.04} & 28.82 \std{0.24}          & 58.10 \std{0.10}          & 62.87 \std{0.13}          \\
Cal~\cite{margatina2021active}                                                  & \textbf{11.58 \std{0.10}}          & 29.76 \std{0.19}          & 37.71 \std{0.04}          & 21.15 \std{0.13}          & 42.84 \std{0.13}          & 51.70 \std{0.03}          & \textbf{35.06 \std{0.22}} & 58.28 \std{0.08}          & 62.61 \std{0.03}          \\
GraphCut~\cite{pmlr-v132-iyer21a}                                             & 10.21 \std{0.21} & 31.88 \std{0.02}          & 41.77 \std{0.15}          & 20.06 \std{0.03}          & 45.50 \std{0.14}          & 53.55 \std{0.11}          & \textbf{35.58 \std{0.37}} & \textbf{59.22 \std{0.15}} & 63.05 \std{0.06}          \\
Forgetting - Hard~\cite{toneva2018an}                                 & 3.19 \std{0.09}           & 22.91 \std{0.31}          & 37.22 \std{0.19}          & 14.65 \std{0.02}          & 42.68 \std{0.13}          & 52.54 \std{0.03}          & 29.55 \std{0.31}          & 57.89 \std{0.03}          & 62.35 \std{0.04}          \\
Forgetting - Easy                                 & 4.12 \std{0.14}           & 17.65 \std{0.17}          & 26.69 \std{0.06}          & 21.63 \std{0.15}          & 46.40 \std{0.06}          & 54.04 \std{0.13}          & 34.26 \std{0.13}          & \textbf{59.06 \std{0.02}} & \textbf{63.23 \std{0.05}} \\
Dyn-Unc~\cite{dyn_unc} & \textbf{12.15 \std{0.22}}          & 33.03 \std{0.02}          & 42.01 \std{0.10}          & 20.94 \std{0.26}          & 45.54 \std{0.15}          & 53.40 \std{0.18}          & 30.39 \std{0.34}          & 58.22 \std{0.07}          & 62.41 \std{0.03} \\
EL2N - Hard~\cite{paul2021deep}                                       & 1.67 \std{0.06}           & 7.00 \std{0.18}           & 12.75 \std{0.09}          & 17.53 \std{0.15}          & 44.12 \std{0.24}          & 52.36 \std{0.04}          & 33.41 \std{0.18}          & 58.16 \std{0.02}          & 62.05 \std{0.06}          \\
EL2N - Best                                          & 10.33 \std{0.39}          & \textbf{38.06 \std{0.08}} & \textbf{47.18 \std{0.17}} & \textbf{22.82 \std{0.41}} & \textbf{48.76 \std{0.56}} & \textbf{56.05 \std{0.12}} & 33.62 \std{0.10}          & \textbf{59.16 \std{0.09}} & \textbf{63.32 \std{0.06}} \\

\midrule
Full Dataset                                                & 41.86      & 62.50 & 65.91          & 41.86      & 62.50 & 65.91          & 41.86      & 62.50 & 65.91          \\
\bottomrule
\end{tabular}}
\end{table*}

Although data-efficient learning broadly encompasses both coreset selection and dataset distillation, methods across these areas are typically evaluated\footnote{Note that by “evaluation”, we mean assessing the quality of a dataset by training a downstream student model on that dataset under a particular label regime, and reporting the final test accuracy using hard labels.} under disparate label supervision settings, making direct comparisons challenging. For instance, large-scale DD methods \cite{yin2024squeezerecoverrelabeldataset,shao2024elucidating,du2024diversitydrivensynthesisenhancingdataset,sun2024diversityrealismdistilleddataset} predominantly report results in the SL+KD regime, where multiple soft labels (one per augmented image) are used, while small-scale DD methods commonly adopt either the hard-label (HL) or fixed soft-label (SL) setting. In contrast, coreset methods evaluate exclusively in the HL setting. This disparity in label supervision also reflects in the highly varying model performances reported in the literature. Therefore, to understand the implications of these different label regimes, we conduct a systematic analysis of each, comparing the performance and scaling behavior of representative DD and coreset methods under each of this settings, and offer insights and recommendations based on our findings. 

For our analysis, we consider \textit{five} large-scale (SRe2L \cite{yin2024squeezerecoverrelabeldataset}, DWA \cite{du2024diversitydrivensynthesisenhancingdataset}, RDED \cite{sun2024diversityrealismdistilleddataset}, D4M \cite{su2024d4mdatasetdistillationdisentangled}, Minimax Diffusion \cite{gu2024efficient}) and \textit{four} small-scale (TM \cite{cazenavette2022datasetdistillationmatchingtraining}, DATM \cite{guo2024losslessdatasetdistillationdifficultyaligned}, DM \cite{zhao2022datasetcondensationdistributionmatching} and DC \cite{zhao2021datasetcondensationgradientmatching}) DD methods. Large-scale methods are evaluated on ImageNet-1K (IN1K) \cite{deng2009imagenet} at the resolution of 224 $\times$ 224 with ResNet-18 as the model architecture. Small-scale methods, on the other hand, are evaluated on TinyImageNet (64 $\times$ 64)~\cite{Le2015TinyIV} with ConvNet-D4 as the model architecture. As baselines, we include state-of-the-art coreset methods matched to the corresponding image resolution for each scale setting. 
Prior work in coreset literature \cite{sorscher2023neuralscalinglawsbeating} suggests that selection should be based on the compute or IPC budgets, with \cite{lee2024selmatcheffectivelyscalingdataset,paul2021deep} showing that carefully chosen subsets with appropriate difficulty level perform the best. 
Accordingly, to ensure fair comparisons, we report results for both the hardest and the optimally hard coresets (as determined by IPC-dependent selection~\cite{lee2024selmatcheffectivelyscalingdataset}) derived from the per-sample scores of these methods. Training and hyperparameter details are provided in the supplementary.

\subsection{Analysis of SL+KD setting: What matters?}

As mentioned earlier, large-scale DD methods have failed to outperform random subset baselines in the SL+KD regime (\citet{qin2024a} and SL+KD panel of \cref{tab:main_in1k}), contradicting results from the coreset literature of data-efficient learning, where high-quality coresets comprehensively outperform random baselines in the HL setting \cite{guo2022deepcorecomprehensivelibrarycoreset, dyn_unc}. To better understand this conflicting observation, we leverage the experimental setup of data-efficient learning from \citet{sorscher2023neuralscalinglawsbeating} and perform a scalability analysis of high-quality coresets, examining \textit{how data quality, size and compute interact to shape the performance scaling behavior} in this regime. For this, we employ EL2N \cite{paul2021deep}, a popular and effective coreset technique that identifies samples of appropriate difficulty using the moving average of their proposed scores \cite{paul2021deep}. Using this metric, we generate subsets of various IPC values ranging from $10$ to $700$ on IN1K\footnote{We restrict our analysis to a maximum IPC of $700$ on ImageNet-1K in order to maintain class-balance, which has been shown to a crucial role in obtaining optimal downstream performance \cite{sorscher2023neuralscalinglawsbeating}.}. Each subset is formed by retaining a fraction $f$ of a larger dataset based on EL2N scores, where $f$ varies from $100\%$ (equivalent to random selection) down to $20\%$, thereby spanning a broad spectrum of subset qualities. This enables us to construct both “easy” and “difficult” subsets in order to study how data quality influences performance across different scales. Given these subsets, we train models on them under four distinct compute budgets, ranging from $2$ to $50$ epochs of full dataset training, covering a sufficiently diverse range to make generalizable observations.

\begin{figure*}
    \centering
    \includegraphics[width=\textwidth]{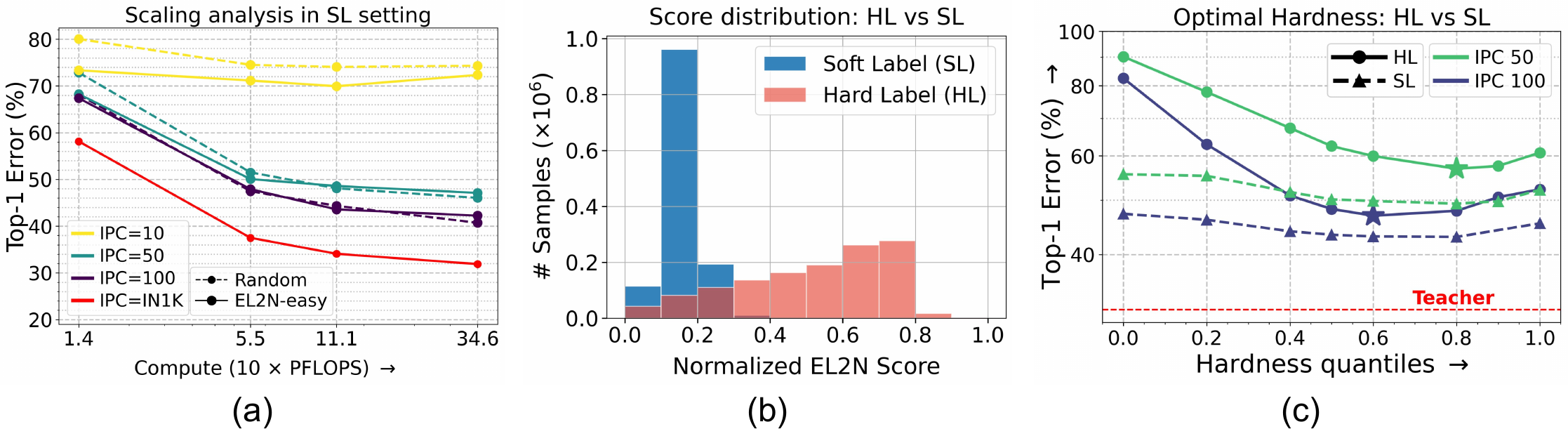}
    \caption{\textbf{Analysis of fixed soft label (SL) setting on ImageNet-1K.} \textbf{(Left)} Performance of coresets of varying quality and size across different compute budgets. While scaling dataset size and compute together remains essential for performance, dataset quality beyond a minimum IPC value play only a minor role as indicated by the convergence of EL2N-easy and random subsets. \textbf{(Middle)} Score distributions during training in SL setting cluster within the easy–mid difficulty range, showing that variations in underlying sample quality have limited effect when trained with fixed soft labels. \textbf{(Right)} Optimal hardness analysis also reveals that performance variations across sets in SL are far smaller compared to the HL setting, reinforcing the diminishing role of sample quality.}
    \label{fig:sl_setting}
\end{figure*}

\vspace{0.1cm}

\noindent \textbf{Results.} The analysis, shown in \cref{fig:motivation}, reveals surprising differences in the scaling behavior of high-quality coresets between the SL+KD regime and the standard HL setting. In the HL setting (see HL plot in \cref{fig:motivation}(a)), improving performance requires jointly scaling dataset size (IPC), its quality (Random vs EL2N-easy) and compute (x-axis). In contrast, under the SL+KD regime (see SL+KD plot in \cref{fig:motivation}(a)), performance is \textit{primarily dictated by compute}, while remaining \textit{largely unaffected by dataset size or quality} for a given compute budget\footnote{We clarify that the term “regardless of quality” does not imply that arbitrary sets (e.g., Gaussian noise) would suffice; it should still belong to the real-image distribution. Likewise, size of subsets beyond a very small IPC value (e.g., IPC 10) offer no additional gains with more compute.}. More strikingly, we observe that the performance level reached by all the subsets of varying quality and size \textit{is near-optimal to the performance of the full dataset for the same compute budget}~\cite{okanovic2024repeated} (shown by red line of IPC=IN1K in the plots). We empirically show the generalizability of these insights to other high-quality coresets beyond EL2N in \cref{tab:main_in1k} (see ``Coreset Selection'' results of the SL+KD panel). We note that such a performance saturation behavior is unique to the setting of SL+KD and therefore raises the question on the scope of progress to be made by continuing to operate in this regime.

\noindent \textbf{Pareto frontier analysis}~\cite{sorscher2023neuralscalinglawsbeating}, which captures the optimal trade-off between dataset size (IPC) and sample difficulty for maximizing performance, further reveals this stark behavioral difference between these two regimes (see \cref{fig:motivation}(b)). In the HL setting, pareto-frontier marked by red markers in the plot shows that for smaller IPC subsets, samples of low difficulty dominate the performance curve (red-bordered squares), gradually shifting towards harder samples at higher IPC (red-bordered triangles). This aligns with our prior understanding of the pareto-optimal trade-offs in data-efficient learning \cite{sorscher2023neuralscalinglawsbeating}. In contrast, under the SL+KD regime, we find that \textit{no such pareto-frontier exists}: beyond the minimum IPC of 10, the performance of all subsets, regardless of difficulty or size, remains largely invariant.

Finally, we also demonstrate this performance saturation behavior as a function of soft labels in \cref{fig:motivation}(c), where one can observe that, \textit{regardless of the quality of the underlying subset (coreset / DD)}, increasing amount of soft-labels in model training (HL $\rightarrow$ fixed soft-label (SL) $\rightarrow$ SL+KD) results in converging performance numbers close to the full dataset performance. 
These findings suggest that continued operating in the SL+KD regime offers limited value due to near-optimal performance saturation across subsets regardless of their underlying quality or size beyond a certain threshold (e.g. IPC 10). Consequently, we recommend eschewing from use of the SL+KD setting for assessing DD methods in favor of evaluation setups that better reveal differences in data quality and enable meaningful progress.

\subsection{The curious case of fixed soft-label (SL) setting}
\label{subsec:main_fixed_sl}
In the SL setting, the soft labels, once obtained from a teacher model on the original dataset, are frozen during the student model training. This setting has gained traction in recent DD methods (including small-scale ones), reporting state-of-the-art results using this label supervision scheme \cite{guo2024losslessdatasetdistillationdifficultyaligned,cui2023scalingdatasetdistillationimagenet1k,qin2024a}. Before we discuss the results of scalability analysis in the SL setting, we first revisit the notion of sample importance, as existing score formulations in coreset methods were originally developed for the HL setting. Given the central role of importance scores in the coreset formation, we first provide a formal treatment of this concept below.

\subsubsection{Sample importance in SL setting}
\citet{paul2021deep} show that the GraNd score (average gradient norm over training) of a sample $x$ can be used to estimate its overall influence on learning. They approximate it using the EL2N metric defined as $\mathbb{E}\,\|p(w_t, x) - y\|_2$, where the expectation is taken over model weights $w_t$ during training and $p$ is the softmax score of the model for the class $y$. However, this approximation assumes training with \textit{hard labels under the cross-entropy loss}. We extend the analysis to the SL setting under the assumption of KL-divergence loss and provide a formal definition of the importance score below.

\begin{definition}
\label{def_el2n_sl}
Let $q = (q_1, q_2, \ldots, q_C)$ be the given target probability distribution (the soft labels from a teacher) and let $T$ be the softmax temperature used for training. Then the importance score of a training sample $(x, y)$ with target probability distribution $q$ is defined as follows: 
\[
\text{EL2N-SL}(x) = \frac{1}{T}\,\mathbb{E}\|p(w_t, x) - q(w_t, x)\|_2
\]
\end{definition}

We defer the derivation to the supplementary material and use this formulation of EL2N-SL score to study the scaling behavior of coresets in the SL setting.

\textbf{Results.} Similar to the SL+KD regime, the scalability analysis of coresets in the SL setting is performed by varying the dataset size, quality, and compute and report the results in \cref{fig:sl_setting}(a). Results reveal a surprising trend: while scaling dataset size and compute together remains essential for optimal performance, dataset quality beyond a minimum IPC value (IPC 10) play only a minor role—as indicated by the convergence of EL2N-easy and random subsets. This observation is consistent for other high-quality coresets as well beyond EL2N, where subsets showing large performance gaps in the HL setting exhibit similar performance under the SL regime (see SL panel of \cref{tab:main_in1k}). 

To investigate why this behavior occurs even with \textit{fixed} soft labels, we first examine the modified EL2N-SL scores (\cref{def_el2n_sl}) for all the sample points and visualize their distribution during training in \cref{fig:sl_setting}(b). Interestingly, the per-sample score distributions differ markedly between the HL and SL settings: while in the HL case, samples are broadly spread across the EL2N-SL spectrum reflecting diverse learnability patterns, in the SL regime, \textit{they are tightly clustered, indicating that most samples contribute similarly to the learning process when trained with fixed soft labels}. This homogenization of sample difficulty is further evident in \cref{fig:sl_setting}(c): unlike the HL setting, where sample hardness substantially impacts performance depending on the score quantile (x-axis) and IPC setting, and where the optimal difficulty (marked by star) shifts towards left as one increases IPC \cite{lee2024selmatcheffectivelyscalingdataset,sorscher2023neuralscalinglawsbeating}, the variations in SL regime are much more muted across scores and IPC, reinforcing the diminishing role of sample quality in this regime as well.

\textbf{Results for small-scale methods.} As mentioned earlier, most small-scale DD methods train models using hard labels \cite{cazenavette2022datasetdistillationmatchingtraining,zhao2022datasetcondensationdistributionmatching}. However, when (fixed) soft labels are used for model training (SL), they are often \textit{jointly learned} with the images during optimization \cite{guo2024losslessdatasetdistillationdifficultyaligned,zhou2022datasetdistillationusingneural}, which can lead to inflated performance gains relative to baseline methods like coresets. Moreover, many-a-times results for baseline methods are still reported in the HL setting, further making the comparison unfair. Therefore, to ensure consistency and transparency, we replace the learned soft labels of small-scale DD sets with teacher-assigned soft labels, mirroring the setup used for baselines. This adjustment isolates the effect of label learning—which, as shown by \citet{qin2024a}, can also be applied to selection-based methods—from the intrinsic quality of the distilled data.


We present the results in \cref{tab:small_scale} (SL panel). While approaches like TM~\cite{cazenavette2022datasetdistillationmatchingtraining} and DATM~\cite{guo2024losslessdatasetdistillationdifficultyaligned} \textit{clearly outperform coreset baselines in the HL setting} for different tasks (in-domain as well as cross-arch. transfer), this advantage disappears when we shift to the SL regime, \textit{even at low IPC settings}. For instance, while DATM outperforms K-centers by $\sim\!8\%$ in the HL setting, it fails to do so in the SL setting. This suggests that, across small- and large-scale settings of dataset distillation and coresets, subset quality, which plays a major role in HL setting, has limited influence in the regime of fixed soft labels. 



\begin{figure*}
    \centering
    \includegraphics[width=\textwidth]{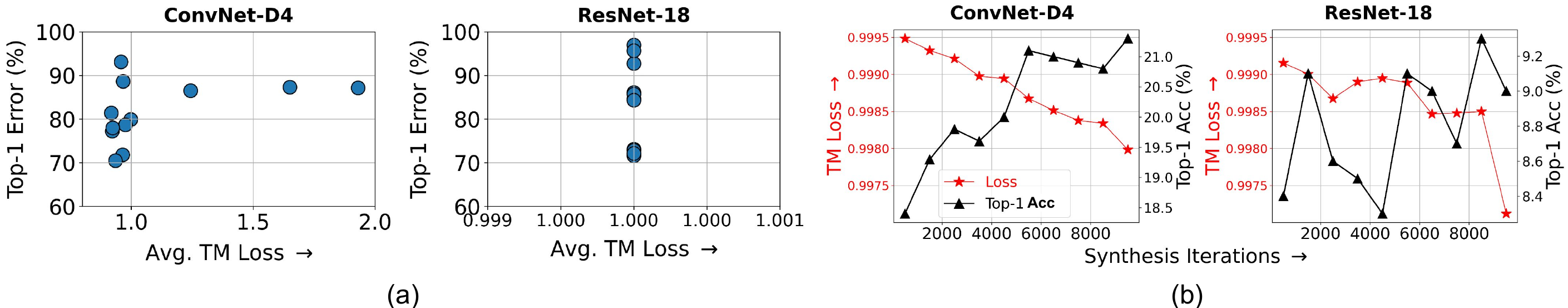}    \caption{\textbf{(Left)} Analysis of TM Loss objective behavior for different synthesis methods on TinyImageNet. Scatter plot displaying correlation of Avg. TM Loss with In-domain generalization of all the methods. Notice the complete lack of correlation when one evaluates the TM loss for larger architectures like RN-18, even though generalization performance varies significantly for the underlying distilled sets. \textbf{(Right)} Training dynamics of DATM synthesis on TinyImageNet. We track DATM synthesis for ConvNet-D4 (left) and ResNet-18 (right) at IPC 10. Note that despite synthesis of $10k$ iterations, both TM loss (red curve) and accuracy (black curve) show minimal change.}
    \label{fig:dcs_tm}
\end{figure*}

\begin{table}
\centering
\caption{\textbf{Performance comparison of small-scale DD methods with coresets on TinyImageNet in HL and SL setting}. Model architecture is ConvNet-D4. The substantial performance gap in the HL setting closes when trained with fixed soft labels.}
\vspace{-0.1cm}
\setlength{\tabcolsep}{3pt}
\resizebox{0.98\linewidth}{!}{
\label{tab:small_scale}
\begin{tabular}{l|cc|cc}
\toprule
\multicolumn{1}{c|}{\multirow{2}{*}{\textbf{Method}}} & \multicolumn{2}{c|}{\textbf{Hard Label (HL)}} & \multicolumn{2}{c}{\textbf{Fixed Soft Label (SL)}}                                   \\
\cmidrule{2-5}
\multicolumn{1}{c|}{}                                 & \textbf{IPC 10}    & \textbf{IPC 50}    & \multicolumn{1}{c}{\textbf{IPC 10}} & \multicolumn{1}{c}{\textbf{IPC 50}} \\
\midrule
\multicolumn{5}{c}{\textbf{Dataset Distillation}} \\
\midrule
DM  \cite{zhao2022datasetcondensationdistributionmatching}                                                 & 13.51 \std{0.31}    & 22.76 \std{0.28}    & 16.12 \std{0.33}                     & 36.58 \std{0.07}                     \\
DC  \cite{zhao2021datasetcondensationgradientmatching}                                                 & 12.83 \std{0.14}    & 12.66 \std{0.36}    & 7.23 \std{0.20}                      & 10.18 \std{0.32}                     \\
TM \cite{cazenavette2022datasetdistillationmatchingtraining}                                                  & \textbf{20.11 \std{0.16}}    & 28.16 \std{0.45}    & 26.11 \std{0.30}                     & 36.43 \std{0.08}                     \\
DATM  \cite{guo2024losslessdatasetdistillationdifficultyaligned}                                               & 19.26 \std{0.19}                 & \textbf{29.51 \std{0.37} }                & \textbf{28.43 \std{0.27}}                     & \textbf{37.39 \std{0.28}}                     \\
\midrule
\multicolumn{5}{c}{\textbf{Coreset Selection}} \\
\midrule
Random Real                                          & 6.88 \std{0.25}     & 18.62 \std{0.22}    & 26.07 \std{0.48}                     & 36.87 \std{0.17}                     \\
K-centers                                               & 11.38 \std{0.26}    & 22.02 \std{0.40}    & 27.11 \std{0.13}                     & \textbf{37.45 \std{0.21}}                    \\
\bottomrule
\end{tabular}}
\end{table}

\subsection{The hard label (HL) setting}
\label{subsec:fixed_hl} 
Given the limited importance of data quality in the soft label regimes of SL and SL+KD, we focus on the HL setting, where the role of data quality is well established~\cite{sorscher2023neuralscalinglawsbeating}. 
Despite this, recent large-scale DD methods have largely moved away from this label setting, often omitting hard-label results altogether from their evaluation. Therefore, a systematic evaluation of current state-of-the-art DD methods (both large- and small-scale) is warranted in the HL setting, which we report in \cref{tab:main_in1k} and \cref{tab:small_scale}, respectively. 

We observe that DD methods at both small- and large-scale outperform random subset baselines across IPC values, with RDED~\cite{sun2024diversityrealismdistilleddataset} being the only method to do so at large scale, and DM~\cite{zhao2022datasetcondensationdistributionmatching}, TM~\cite{cazenavette2022datasetdistillationmatchingtraining}, and DATM~\cite{guo2024losslessdatasetdistillationdifficultyaligned} at small scale. Among these, only RDED, TM, and DATM perform competitively with coreset methods beyond random subset selection. However, due to the limited scalability of TM and DATM, their evaluations are restricted to TinyImageNet with ConvNet-D4 as the model. Although they are not directly scalable to large-scale settings such as IN1K, their underlying distillation objectives can still be studied to assess their scaling behavior. To do this, we introduce a new evaluation framework in the following section that performs this analysis efficiently and in a zero-shot manner without the need for the costly synthesis process.

\section{Distillation Correlation Score (DCS)}

Recall from our preliminaries that a distillation method aims to optimize a surrogate loss objective $\mathcal{L}_{\text{distill}}(\mathcal{S}, \mathcal{D})$ which yields a distilled set $\mathcal{\hat{S}}^*$ closely approximating $\mathcal{S}^*$, the optimal synthetic set of the original dataset $\mathcal{D}$. Therefore, to assess the quality of a distilled set, one can instead analyze the generalization behavior of its underlying distillation objective $\mathcal{L}_{\text{distill}}$ used to synthesize it in the first place. Shifting the focus to the objective enables us to evaluate a wide-range of objective functions to assess their potential for a better distillation objective. To do this, we start with $m$ different sets \( \{S_1, \dots, S_m\} \), where $S_j$ can be either a coreset or a distilled set. For each set $S_j$, we train a model $\theta_{\mathcal{S}_{j}}^*$ and evaluate its loss on the downstream test set $\mathcal{D}_{test}$ ($\ell_{\mathcal{D}_{test}}(\theta_{\mathcal{S}_{j}}^*)$). We also compute its distillation loss objective $\mathcal{L}_{\text{distill}}(\mathcal{S}_{j},\mathcal{D})$. Using these two losses, we propose \textbf{Distillation Correlation Score (DCS)} for a given dataset distillation method $\phi$ on a downstream test set $\mathcal{D}_{test}$:
\[
\text{DCS}(\phi) = \mathlarger{\rho}\left( \left\{ \ell_{\mathcal{D}_{test}}(\theta_{\mathcal{S}_{j}}^*) \right\}_{j=1}^m, \Bigl\{ \mathcal{L}_{\text{distill}}(\mathcal{S}_{j},\mathcal{D}) \Bigl\}_{j=1}^m \right)
\]
where $\rho$ denotes the Spearman-rank correlation. The DCS score offers an efficient way to evaluate different distillation objectives: since the generalization errors for various sets $\mathcal{S}_j$ are computed only once by training a student model on them, they can be reused as a look-up table to compute correlation with loss values of any (existing or future) distillation objective $\mathcal{L}_{\text{distill}}$. Therefore, using DCS, we not only mitigate the compute-intensive process of generating distilled sets, but we can also use it to identify objectives which are e aligned with the downstream generalization, serving as better indicators of distillation quality.

\subsection{Loss Objective evaluation using DCS}
\label{subsec:eval_using_dcs}

Using DCS, we analyze the scaling behavior of two promising small-scale methods mentioned in the last section: TM \cite{cazenavette2022datasetdistillationmatchingtraining} and DATM \cite{guo2024losslessdatasetdistillationdifficultyaligned}, and assess whether scaling them would help improve generalization on larger models like ResNet(RN)-18/50. Briefly, TM synthesizes datasets by aligning optimization trajectories of models trained on real and synthetic data, and DATM extends this by constraining the matching range and jointly optimizing soft labels with the images. Owing to the shared mechanisms behind the workings of the methods, we therefore focus on their underlying synthesis objective and report the results in \cref{fig:dcs_tm}.

A concerning behavior of the TM objective emerges: We observe \textit{no correlation between the TM surrogate loss and generalization when using a larger architecture such as RN-18 for synthesis}. Although RN-18 models trained on different subsets exhibits considerable variation in performance (see Top-1 Error), their TM surrogate loss remains unchanged at approximately $0.999$ (see \cref{fig:dcs_tm}(a)). In contrast, we observe reasonable variations in the TM surrogate loss values for smaller architectures like ConvNet-D4 which is commonly employed in TM-based methods. This indicates that, for larger models such as RN-18, the TM objective fails to capture meaningful variations between subsets, which is essential for an effective synthesis. 

To further ascertain the above observations, we perform synthesis using the DATM method for TinyImageNet using both ConvNet-D4 and ResNet-18 architectures, and track its TM surrogate loss dynamics. From \cref{fig:dcs_tm}(b), we note that despite $10k$ synthesis iterations, one can hardly observe any significant change in both loss and accuracy when using RN-18 for synthesis.
This further confirms our observations that using TM-based loss objective for synthesis in larger models and datasets is not useful, as it's unable to capture any meaningful differences between subsets, which is essential for downstream generalization. We provide more results on DCS in the supplementary material demonstrating its robustness across different DD objectives, tasks (in-domain vs out-of-domain performance), and on other DD methods (small- and large-scale), observing a lack of any \textit{strong} correlation of the objectives with generalization.


\section{Proposed Method}
\label{sec:method}

Given the lack of scalability of TM-based objectives, we turn our focus to RDED~\cite{sun2024diversityrealismdistilleddataset} in the large-scale setting, which convincingly outperforms random selection across all IPC in the HL setting but fails to convincingly surpass best coreset methods, potentially due to the simple underlying heuristic of selecting easy sample patches to construct its distilled set. As noted in coreset literature \cite{paul2021deep,sorscher2023neuralscalinglawsbeating,lee2024selmatcheffectivelyscalingdataset}, sample difficulty plays a crucial role in determining downstream model performance in the HL setting. Prior works such as EL2N~\cite{paul2021deep} and Dyn-Unc~\cite{dyn_unc} estimate per-sample difficulty via training-time statistics such as prediction error or uncertainty, but these fixed scores fail to generalize across IPC budgets (see EL2N-Hardest and Dyn-Unc in \cref{tab:main_in1k}), as easy (hard) samples are preferable under low (high) IPC or compute budgets. To address this, \citet{lee2024selmatcheffectivelyscalingdataset} proposed a “sliding window” approach that searches over all difficulty ranges of scores to identify subsets optimal for a given IPC. While effective, this method is computationally prohibitive as it requires multiple full training runs for each candidate window e, rendering it impractical for large-scale datasets or higher IPC values.

We therefore propose \textbf{CAD-Prune}, a compute-aware dataset pruning metric which generalizes effectively across different IPC / compute budgets while remaining computationally efficient. Building on the insight from Dyn-Unc~\cite{dyn_unc} that score uncertainty is a more reliable indicator of sample importance than the mean, CAD-Prune estimates per-sample uncertainty by training a model on the full dataset using the same compute budget as the downstream task (e.g. 200 epochs of IPC 50 set). 
To obtain the difficulty estimates optimal for that compute budget, we average the uncertainty scores not over the full training run, but over a small window near the end of the compute-dictated training period. This captures the samples that are \textit{still being learned under that training budget}. Crucially, the compute-matched run must follow the same learning-rate schedule as the full run; otherwise, we fail to replicate which samples remain learnable at that stage. This also means that using early checkpoints from a full run is insufficient, as they miss samples that would only be learned during the low–learning-rate phase of the compute-limited training.

\begin{table}
\centering
\caption{\textbf{Performance comparison of the proposed method on ImageNet-1K in HL setting}. We compare the proposed compute-optimal coreset CAD-Prune against EL2N-Best obtained using the sliding-window approach of \citet{lee2024selmatcheffectivelyscalingdataset}, and the proposed DD method CA2D against RDED. Both the methods outperform their best counterparts on ImageNet-1K in HL setting.}
\vspace{-0.2cm}
\setlength{\tabcolsep}{8pt}
\resizebox{\linewidth}{!}{
\label{tab:method}
\begin{tabular}{lccc}
\toprule
\multicolumn{1}{c}{\textbf{Method}} & \textbf{IPC 10} & \textbf{IPC 50} & \textbf{IPC 100} \\
\midrule
\multicolumn{4}{c}{\textbf{Coreset Selection (224 $\times$ 224)}}                                             \\
\midrule
EL2N-Best                           & 10.53 \std{0.14}           & 38.44 \std{0.30}           & \textbf{47.76 \std{0.43}}            \\
CAD-Prune                           & \textbf{12.57 \std{0.03}}  & \textbf{40.21 \std{0.15}}  & \textbf{47.40 \std{0.20}}   \\
\midrule
\multicolumn{4}{c}{\textbf{Dataset Distillation (224 $\times$ 224)}}                                          \\
\midrule
RDED                                & 14.34 \std{0.17}          & 38.49 \std{0.41}           & 44.36 \std{0.07}           \\
CA2D                                & \textbf{15.25 \std{0.24}}  & \textbf{41.72 \std{0.40}}  & \textbf{46.32 \std{0.10}}              \\
\bottomrule
\end{tabular}}
\end{table}

We now formally define our proposed metric CAD-Prune. Given a model $\theta^k$ trained on the whole dataset at the $k$-th epoch, 
a sample's score (we consider EL2N score here) is denoted as $\mathbb{S}(y~|~x, \theta^k)$. The uncertainty is defined as the standard deviation of the scores 
in successive $J$ training epochs with models 
$\{\theta^k, \theta^{k+1}, \ldots, \theta^{k+J-1}\}$:
\begin{equation}
U_k(x) = 
\sqrt{
\frac{\sum_{j=0}^{J-1} 
    \left[ 
        \mathbb{S}(y|x, \theta^{k+j}) - \bar{\mathbb{S}} 
    \right]^2
}{
J - 1
}}
\end{equation}
where $\bar{\mathbb{S}} = \frac{1}{J}\sum_{j=0}^{J-1} \mathbb{S}(y|x, \theta^{k+j}),$ and $J$ is the range for calculating the deviation. These epoch-wise uncertainty scores are used to calculate the final per-sample CAD-Prune scores given by:
\begin{equation}
CAD(x) = 
\frac{
\sum_{k=K-J-W}^{K-J-1} U_k(x)
}{
W
}.
\end{equation} 
Here, $W$ is the small window size we average the uncertainty over.
We typically set $W=2$ and $J=6$ for IN1K.

Using CAD-Prune, we first identify samples that correspond to an appropriate difficulty level for a given compute budget. We then crop, select, and stitch patches in the RDED-style from this coreset to construct a distilled set tailored to that compute level. We refer this as \textbf{CA2D (Compute-Aware Dataset Distillation)}.

We show results on ImageNet-1K in the HL setting in \cref{tab:method}. Our adaptive compute-aware pruning method CAD-Prune outperforms the costly sliding-window baseline while being far more efficient. The sliding-window approach first computes EL2N scores over the full training run and then evaluates multiple window configurations, each requiring an extra training cycle. In contrast, our method yields budget-optimal scores with just a single compute-matched run. For example, for IPC 100 coreset selection, this delivers 2–4$\times$ compute savings despite not accounting for the overhead of generating scores in the first place from the full run. CA2D also outperforms RDED on ImageNet-1K, setting a new state-of-the-art in the HL setting.

\section{Conclusion}
\label{sec:conclusion}

In this work, we systematically study how dataset quality, size, and compute interact across key label regimes in data-efficient learning across both coresets and DD, showing that soft-label supervision substantially diminishes the importance of data quality and that multiple soft labels further reduce sensitivity to dataset size, with performance primarily driven by compute and approaching near-optimal full-data levels. In contrast, data quality still holds importance in the hard-label (HL) setting, and further analysis of DD methods under this setting using the proposed DCS metric reveals scalability limitations of TM-based DD methods to larger architectures and datasets. Finally, we propose CAD-Prune, a \textit{compute-aware} pruning metric, and leverage it to develop CA2D, a \textit{compute-aligned} DD method, both of which outperform respective prior coreset and DD baselines on ImageNet-1K across multiple IPC settings. 


\section{Future work and limitations}
Our findings carry broad implications. Since dataset size and quality have limited impact under extensive teacher supervision, this suggests that in model-stealing scenarios, a small dataset with repeated supervision could achieve near-optimal performance. Furthermore, although our study focuses primarily on image classification, future work could explore multimodal setups and tasks beyond classification. Finally, expressing non-optimizable DD objectives for evaluation under DCS requires care and is an open problem.

\vspace{-0.1cm}

\noindent\makebox[\linewidth]{\rule{\linewidth}{0.1pt}}

\noindent \textbf{Acknowledgment.} This work is supported by a grant from Google and Kotak IISc AI-ML Centre (KIAC).
 
{
    \small
    \bibliographystyle{ieeenat_fullname}
    \bibliography{main}
}

\clearpage
\newpage
\appendix
\section*{Appendix}


\section{Details on Methods and their loss objectives}
\label{sec:related}


\subsection{Large-scale methods}
\label{subsec:large_scale_method_desc}
Early dataset distillation methods  \cite{wang2020datasetdistillation, wang2022cafelearningcondensedataset, cazenavette2022datasetdistillationmatchingtraining, zhao2022datasetcondensationdistributionmatching,zhou2022datasetdistillationusingneural} relied on a bi-level optimization framework, where the distilled set is optimized in an outer loop while a model is trained on this distilled set in an inner loop. This process is computationally expensive, as it requires repeated model training for every update of the distilled dataset, making it particularly challenging to scale to larger settings. SRe2L \cite{yin2024squeezerecoverrelabeldataset} addressed this limitation by decoupling this optimization process of the distilled set from the model training. This approach eliminates the need for an inner loop during training, as the dataset is now synthesized by matching certain statistical properties of a pre-trained teacher model. This decoupling enables it to scale to larger datasets like ImageNet-1K \cite{deng2009imagenet} and beyond. Many subsequent works, like EDC \cite{shao2024elucidating}, DWA \cite{du2024diversitydrivensynthesisenhancingdataset}, G-VBSM \cite{shao2024generalizedlargescaledatacondensation}, etc. have adopted a similar approach by building on top of this framework. We briefly describe such techniques below for ease of reference and clarity:

\textbf{SRe2L} \cite{yin2024squeezerecoverrelabeldataset} performs distillation by optimizing two loss objectives: (1) the standard Cross-entropy loss, and (2) another loss which matches the BatchNorm-mean and BatchNorm-variance of the teacher model which is pre-trained on the original dataset. Following is the formulation for the latter part:
\begin{equation}
\begin{aligned}
    \mathcal{L}_{\text{BN-SRe2L}}(\tilde{x}) &=  
    \underbrace{\sum_{l} \left\| \mu_l(\tilde{x}) - \mathbf{BN}_l^{\text{RM}} \right\|_2}_{\text{BN-Mean loss}} \\ 
    &\quad +
    \underbrace{\sum_{l} \left\| \sigma^2_l(\tilde{x}) - \mathbf{BN}_l^{\text{RV}} \right\|_2}_{\text{BN-Var loss}}
\end{aligned}
\end{equation}
Here, \( l \) is the index of the Batch Normalization (BN) layer, \( \mu_l(\tilde{x}) \) and \( \sigma^2_l(\tilde{x}) \) represent the mean and variance of the model trained on distilled set, respectively. \( \mathbf{BN}_l^{\text{RM}} \) and \( \mathbf{BN}_l^{\text{RV}} \) denote the running mean and running variance statistics of the pre-trained model at the \( l \)-th layer. 


\textbf{DWA} \cite{du2024diversitydrivensynthesisenhancingdataset} (Directed Weight Adjustment) aims to address the diversity limitations of SRe2L by introducing \textit{weight perturbations} in the model during matching along with a separate hyperparameter $\boldsymbol{\lambda_{var}}$ for the variance matching loss term: 
\begin{equation}
\begin{aligned}
    \mathcal{L}_{\text{BN-DWA}}(\tilde{x}) &=  
    \text{BN-Mean loss}(\tilde{x}) \\
    &\quad + \boldsymbol{\lambda_{var}} \cdot
    \text{BN-Var loss}(\tilde{x})
\end{aligned}
\end{equation}

\textbf{RDED} \cite{sun2024diversityrealismdistilleddataset} creates a synthetic dataset by extracting and concatenating the most "impactful" patches (where the impact is measured using a scoring mechanism using the output of a pretrained teacher model) from all the images of a particular class. Therefore, even though RDED is considered a "synthesis" method by the research community, but in its core essence, it's a technique which groups together informative patches \textit{of real images} to construct a compact set.

\textbf{D4M} \cite{su2024d4mdatasetdistillationdisentangled} is a generative distribution matching method that utilizes a pre-trained Latent Diffusion Model (LDM) \cite{rombach2022high} to generate synthetic images. Multiple prototypes are created for each class using K-means clustering in the latent space, which are then denoised using the pre-trained LDM model before passing them through a pre-trained decoder to produce synthetic images.

    \textbf{Minimax Diffusion} \cite{gu2024efficient} incorporates diffusion-transformer (DiT) based generative models to create a distilled set. They do so by finetuning a pretrained DiT model separately for each subset of classes in order to learn their data distribution. This helps ammortize the generation cost for higher IPCs by simply performing inference with the finetuned model multiple times for class-conditional generation. Their loss objective involves three loss terms: 
    
        (1) Standard \textit{diffusion loss} $\mathcal{L}_{\text{simple}} = || \epsilon_{\theta}(\boldsymbol{\mathrm{z_t}},~\boldsymbol{\mathrm{c}}) - \epsilon||^2_2$ of predicting the ground truth noise $\epsilon$ using the noise prediction network $\epsilon_{\theta}$ at a time step $\text{\textbf{t}}$; 
    
    (2) a \textit{minimax loss} $\mathcal{L}_r = \text{arg}\max_{\theta} \min_{m ~\in ~[N_m]} \cos (\boldsymbol{\mathrm{\hat{z}}}_{\theta}(\boldsymbol{\mathrm{z_t}},~\boldsymbol{\mathrm{c}}), \boldsymbol{\mathrm{z}}_m)$, which pulls close together the predicted embedding $\boldsymbol{\mathrm{\hat{z}}}_{\theta}$ to the \textit{least} similar sample $\boldsymbol{\mathrm{z}}_m$ stored in the memory bank $\mathcal{M} = \{\boldsymbol{\mathrm{z}}_m\}_{m=1}^{N_m}$ to ensure representative alignment; and 
    
    (3) a \textit{sample diversity loss} $\mathcal{L}_d = \text{arg}\min_{\theta} \max_{d ~\in ~[N_D]} \cos (\boldsymbol{\mathrm{\hat{z}}}_{\theta}(\boldsymbol{\mathrm{z_t}},~\boldsymbol{\mathrm{c}}), \boldsymbol{\mathrm{z}}_d)$, which has an opposite optimization target compared to the minimax loss $\mathcal{L}_r$, in where the predicted embedding $\boldsymbol{\mathrm{\hat{z}}}_{\theta}$ is pushed away from the \textit{most} similar one stored in another memory bank $\mathcal{D} = \{\boldsymbol{\mathrm{z}}_d\}_{d=1}^{N_D}$. 
    
    Thus, the overall loss term becomes: 
\begin{equation}
\mathcal{L}_{\text{Minimax}} = \mathcal{L}_{\text{simple}}~ + ~ \lambda_r~ \mathcal{L}_r ~+~ \lambda_d~\mathcal{L}_d
\end{equation}

where $\lambda_r$ and $\lambda_d$ are weighting hyper-parameters.

While these methods utilize various (complex) techniques for image synthesis itself, they share a common strategy of employing multiple soft labels for every augmented image for downstream training of a student model. These soft labels provide a more detailed supervisory signal to guide the student model's training and have been shown to significantly improve the downstream performance. Apart from this, there are generative-model based distillation methods like D4M \cite{su2024d4mdatasetdistillationdisentangled} and GLaD \cite{cazenavette2023glad} who perform distillation in the latent space and generates distilled images uisng pretrained decoders. For more details, we refer the reader to \citet{li2022awesome}.

\subsection{Small-scale DD methods}
\label{subsec:small-scale-related}
Small-scale dataset distillation methods synthesize compact datasets through a bi-level optimization framework to enable efficient model training with fewer samples. However, scaling these methods to larger datasets such as ImageNet-1K and to deeper architectures like ResNet-18 remains challenging due to their substantial computational cost. In general, methods in this regime fall into three categories: trajectory matching (TM \cite{cazenavette2022datasetdistillationmatchingtraining} and DATM \cite{guo2024losslessdatasetdistillationdifficultyaligned}) , distribution matching (DM \cite{zhao2022datasetcondensationdistributionmatching}), and gradient matching (DC \cite{zhao2021datasetcondensationgradientmatching}). We briefly describe these methods and their loss objectives  below:

\textbf{TM} \cite{cazenavette2022datasetdistillationmatchingtraining} aims to synthesize datasets by aligning the optimization trajectories of models trained on real and synthetic data. Specifically, given a real dataset $\mathcal{D}_{\text{train}}$ and a synthetic dataset $\mathcal{S}$, TM updates $\mathcal{S}$ to minimize a trajectory matching loss. This loss compares the model parameters obtained from training on $\mathcal{S}$ for $N$ steps to those obtained from training on $\mathcal{D}_{\text{train}}$ for $M$ steps, starting from a shared initialization point $\theta_t$:
\begin{equation}
\mathcal{L}_{TM}(\mathcal{S}, \mathcal{D}_{\text{train}}) = 
\frac{\left\| \hat{\theta}_{t+N} - \theta_{t+M} \right\|_2^2}{\left\| \theta_{t} - \theta_{t+M} \right\|_2^2}
\end{equation}
where $\theta_{t}$ denotes the model parameters at step $t$ obtained from training on the real dataset from random initialization, and $\hat{\theta}_{t+N}$ represents the parameters after training on the synthetic dataset for $N$ steps starting from $\theta_{t}$. This objective encourages synthetic data to induce training dynamics similar to that of real data over a fixed training horizon. TM also uses an upper bound $T^+$ on the sampling range of $t$, so that only model parameters within this range are used for matching. In the original work, TM uses 100 different expert trajectories for synthesis. These expert trajectories are obtained by training on the real dataset with a different random initialization and saving intermediate checkpoints along the way. 

\textbf{DATM} \cite{guo2024losslessdatasetdistillationdifficultyaligned} improves upon Trajectory Matching (TM) with two main modifications: (1) First, it sets both a lower bound $T^-$ and an upper bound $T^+$ on the sampling range for $t$, so that only parameters between these bounds are used for matching; and (2) second, it includes soft labels in the distilled data, which are also optimized during the distillation process.

\textbf{DM} \cite{zhao2022datasetcondensationdistributionmatching}
learns a distilled dataset by directly matching the output features of real and synthetic samples, where the features are extracted from a model with randomly initialized weights. The loss objective is to minimize the difference between the average feature representations of real and synthetic data under various augmentations. This is formalized as following:


\begin{equation}
\begin{aligned}
\min_{S} \ \mathbb{E}_{v \sim \mathcal{P}_v,\ \omega \sim \Omega} 
\Bigg\|
    \frac{1}{|T|} \sum_{i=1}^{|T|}
        \psi_v\!\left(A(x_i,\omega)\right) 
\\[-0.4em]
    -\ 
    \frac{1}{|S|} \sum_{i=1}^{|S|}
        \psi_v\!\left(A(\tilde{x}_i,\omega)\right)
\Bigg\|^{2}
\end{aligned}
\end{equation}

\begin{table*}
\centering
\caption{\textbf{Hyperparameter settings for evaluating performance of distilled set}. Below hyperparameters are used to train a student model on a distilled set to evaluate its generalization performance on downstream tasks. KL-Div: Kullback-Leibler Divergence; T: Temperature of KL-Div loss; RRC: Random Resized Crop, HFlip: Random Horizontal Flip, AdamW: Adam optimizer with decoupled Weight decay, Cosine: Cosine annealing; DSA: Differentiable Siamese Augmentation.}
\vspace{0.1cm}
\setlength{\tabcolsep}{6pt}
\resizebox{\textwidth}{!}{
\label{tab:stage3_hyper}
\begin{tabular}{l|p{5cm}p{5cm}|p{2.5cm}p{2.5cm}}
\toprule
\multicolumn{1}{c|}{\multirow{2}{*}{\textbf{Hyperparameter}}} & \multicolumn{2}{c|}{\textbf{Large-scale (ImageNet-1K)}}      & \multicolumn{2}{c}{\textbf{Small-scale (Tiny-IN / CIFAR-100)}}                                   \\
\cmidrule{2-5}
\multicolumn{1}{c|}{}                                         & \multicolumn{1}{l}{\textbf{Hard Label (HL)}} & \multicolumn{1}{l|}{\textbf{Soft Label (SL / SL+KD)}}          & \multicolumn{1}{l}{\textbf{HL setting}} & \multicolumn{1}{l}{\textbf{SL setting}} \\
\midrule
Epochs                                                       & 200                  & 200                            &    300                                       &   300                                     \\
Loss function & Cross-Entropy (CE) & KL-Div ($T$=$20$) & Cross-Entropy (CE) & KL-Div ($T$=$20$) \\
Optimizer                                                    & SGD                & AdamW                          &   SGD                                        &   AdamW                                     \\
Learning rate                                                & 0.5                 & 1e-3                           &  1e-2                                         &  1e-3                                      \\
Scheduler                                                    & Cosine               & Cosine                         &   StepLR@151                                        & Cosine                                       \\
Batch size                                                   & 50 (IPC 10), 128 (IPC 50), 200 (IPC 100)                  & 128                            &      256                                     &   256                                     \\
Augmentations                                                & RRC + HFlip (+ PatchShuffle aug. for RDED-type sets) & RRC + HFlip (+ PatchShuffle aug. for RDED-type sets) (+ Cutmix aug. for SL+KD) &  DSA                                          &  DSA                                      \\
Other details                                                & --                   & Warm-up of 5 epochs                  &  --                                          &  -- \\
\bottomrule
\end{tabular}}
\end{table*}

where $\psi_v$ is a  function mapping inputs to a randomly initialized model's feature space, $\omega$ is a data augmentation parameter sampled from a distribution $\Omega$, and $A(\cdot, \omega)$ denotes the augmented version of an input. $T$ is the set of real samples, and $S$ is the synthetic dataset being optimized.

\textbf{DC} \cite{zhao2021datasetcondensationgradientmatching} matches the gradients of the loss function computed on real and synthetic data with respect to the network parameters $\theta$. The synthetic dataset $S$ and the model parameters $\theta$ are optimized in an alternating manner, typically under a bilevel optimization framework. The DC loss objective is formulated as follows:



\begin{equation}
\begin{aligned}
S^{*} 
= \arg\min_{S} \ 
    \mathop{\mathbb{E}}\limits_{\theta_{0} \sim \mathcal{P}_{\theta_{0}}}
    \Bigg[
        \sum_{t=0}^{T-1}
        \mathcal{D}\!\left(
            \nabla_{\theta}\mathcal{L}_{S}(\theta_{t}),
            \nabla_{\theta}\mathcal{L}_{T}(\theta_{t})
        \right)
    \Bigg]
\\[0.3em]
\text{subject to} \qquad
\theta_{t+1}
    = \operatorname{opt\_alg}_{\theta}
        \!\left(\mathcal{L}_{S}(\theta_{t}),
        \varsigma_{\theta},\,
        \eta_{\theta}\right).
\end{aligned}
\end{equation}

where $\mathcal{P}_{\theta_0}$ is the distribution over network initializations, $T$ is the number of outer-loop iterations used to update the synthetic data, $\varsigma_\theta$ is the number of inner-loop optimization steps, $\eta_\theta$ is the learning rate for the model parameters, and $\mathcal{D}(\cdot, \cdot)$ measures the distance between real and synthetic gradients. 

\subsection{Coreset selection}
Similar to dataset distillation, coreset selection aims to identify a small, informative subset of the original dataset that, when used for training, achieves performance comparable to training on the full dataset. Broadly, coreset methods can be grouped into the following categories:

1) \textit{Diversity-Based Selection}: These methods aim to select subsets that span the data distribution well, encouraging representativeness and coverage of the input space. Example methods are Facility Location, K-Centers, etc. Facility Location uses pairwise similarity to select diverse examples that “cover” the dataset, whereas K-Centers selects points that are maximally distant from each other or represent cluster centers.


2) \textit{Importance or Influence-Based Selection}: These approaches estimate the influence or importance of each data point based on its effect on model training or generalization. For instance, Craig \cite{mirzasoleiman2020coresets} selects points to approximate the gradient of the full dataset. Glister \cite{killamsetty2021glistergeneralizationbaseddata}, on the other hand, uses bilevel optimization to select subsets that maximize validation performance.

3) \textit{Uncertainty and Informativeness-Based Selection}: These methods select examples that are likely to be the most informative or uncertain, often inspired by active learning. One example in this category is Entropy or Least Confidence, which selects examples with high predictive uncertainty. Another example is Margin Sampling, which picks points close to the decision boundary.

4) \textit{Learning Dynamics and score-based Selection}: Approaches in this category leverage model training behavior to identify informative or difficult examples. For example, Forgetting \cite{toneva2018an} tracks how often samples are forgotten during training, whereas GraNd / EL2N \cite{paul2021deep} ranks samples using the gradient / error norm of the loss over a few training epochs.


DeepCore \cite{guo2022deepcorecomprehensivelibrarycoreset} provides a comprehensive benchmarking of these methods across different datasets and tasks, highlighting their strengths and limitations in practical settings.

\section{Details on Training and Hyperparameters}
\label{sec:hparam_details}
We describe relevant hyperparameters for evaluation of both small-scale and large-scale DD methods in this section. Note that the hyperparameters for the teacher training and synthesis stage vary depending on the approach utilized. 
As we perform synthesis of any distilled set using their official codebase and recommended hyper-parameter settings (and use pretrained distilled sets when available), we mainly provide the details of the student training on distilled sets, which are summarized in \cref{tab:stage3_hyper}. EDC \cite{shao2024elucidating} explores the design space of large-scale dataset distillation methods in SL+KD regime and highlights the issue of sub-optimal hyper-parameter settings in training a student model on a distilled set. Subsequently, it proposes a series of hyperparameter changes to extract the best performance out of a synthesized dataset. For instance, it proposes optimal learning rates, batch sizes and scheduler which substantially improves student performance on various downstream tasks. Therefore, we adopt EDC's hyperparameter configuration of large-scale methods and use it in the SL / SL+KD setting to enable an optimal comparison of different methods. For the HL setting, we tune the optimal hyper-parameters for coresets as well as for the distilled sets around the settings taken from DeepCore \cite{guo2022deepcorecomprehensivelibrarycoreset}. Additionally, we tune for the training batch sizes as per RDED \cite{sun2024diversityrealismdistilleddataset}.


\noindent \textbf{DATM / TM hyperparameters.} We also specify hyperparameters for DATM / TM synthesis experiments on TinyImageNet for DCS. We use the DATM \cite{guo2024losslessdatasetdistillationdifficultyaligned} variant of TM loss for analysis, which sets a lower bound $T^-$ along with an upper bound $T^+$ on the sampling range for the starting checkpoint for matching trajectories. Recall from \cref{subsec:small-scale-related} of the supplementary that the TM loss has two additional hyperparameters $M$ and $N$. For our loss objective analysis experiments (DCS), we set $N=80$ and $M=2$. For synthesis, we use $100$ expert models, each trained for $50$ epochs. Remaining hyperparameters are summarized in \cref{tab:datm_dcs_hyperparams}.

 \begin{table}[]
\centering
\caption{\textbf{Hyperparameters used for DATM synthesis.}}
\label{tab:datm_dcs_hyperparams}
\begin{tabular}{lcccccccc}
\toprule
 \textit{N} & \textit{M} & \textit{$T^-$} & \textit{$T^+$} &  \begin{tabular}[c]{@{}c@{}}Synthetic\\ Batch Size\end{tabular} & \begin{tabular}[c]{@{}c@{}}Learning Rate\\ (Label / Pixels)\end{tabular} \\
 \midrule
 80         & 2          & 20             & 35                    & 250                                                                       & 10 / 100                                                                         \\
 \bottomrule
\end{tabular}
\end{table}
 


\section{Fitting scaling law equation for label regimes}
To quantify the interplay between dataset quality, size, and compute, we adopt the data-aware scaling law introduced by ~\citet{goyal2024scaling}. The predicted performance at epoch $k$ is modeled as
\begin{align}
&y_k = a \cdot n_1^{b_1} \prod_{j=2}^{k} \left( \frac{n_j}{n_{j-1}} \right)^{b_j} + d , \\
&b_{k+1} = b \cdot \left( \frac{1}{2} \right)^{\frac{k}{\tau}} = b \cdot \delta^{k}
\label{eq:scaling_law}
\end{align}
Here, $a$ is a normalizing factor; $b_j$ governs diminishing gains as more data is seen and also reflects the intrinsic utility of the data pool, with lower $b$ indicating higher utility; $\delta$ captures the decay in data utility with repetition; $d$ represents the irreducible loss that cannot be further reduced; and $\tau$ indicates how the utility of a sample decays slowly with repetition.

We fit this scaling law to performance curves of IPC 100 subsets of random and EL2N-Easy under two label regimes: SL+KD and HL. In the SL+KD regime, we find that both the subsets, which vary highly in their underlying quality,  exhibit nearly identical scaling behavior, with $b \approx 0.05$ and $\delta = 2.8947$. In contrast, the HL regime reveals a strong dependence on data quality: random subsets yield $b = -0.1859$ and $\delta = 5.2632$, while EL2N subsets exhibit $b = -0.1558$ and $\delta = 1.9474$, reflecting slower saturation of informative samples.

\textbf{Implication.} These scaling dynamics results indicate that the label regime of SL+KD, where abundant soft labels from a teacher model is used, the supervision nullifies the effect of data utility on downstream performance. This is in stark contrast to the hard-label (HL) regime, where the contribution of sample quality continues to strongly influence the model performance as evidenced by differing values of $b$ and $\delta$.

\section{Derivation of EL2N-SL Score}
\label{sec:proof-el2n-sl}


\newtheorem{assumption}{Assumption}
\newtheorem{lemma}{Lemma}

\begin{assumption}[Per-logit gradient orthogonality \cite{paul2021deep}]
\label{asm:orth}
For a model at time $t$ and input $x$, let $\psi_t^{(k)}(x)=\nabla_{w_t} f^{(k)}_t(x)$ be the parameter gradient of the $k$-th logit. We assume
\[
\langle \psi_t^{(k)}(x), \psi_t^{(j)}(x)\rangle \approx 0\quad (k\ne j),
\quad
\|\psi_t^{(k)}(x)\|_2 \approx c,\ \forall k,
\]
for some constant $c>0$ independent of $k$.
\end{assumption}

\begin{lemma}[] 
\label{lem:grand-el2n}
Let $p(w_t,x)$ be the temperature-$T$ softmax of logits $f(w_t,x)$,
\[
p_i=\frac{\exp(f_i/T)}{\sum_j\exp(f_j/T)},
\]
and let $q$ be a target (soft) distribution. For notation brevity, let $f_t := f(w_t, .)$. Under Assumption~\ref{asm:orth}, the GraNd score
\[
\chi_t(x,q)=\mathbb{E}\Bigg\Vert\sum_{k=1}^C \big(\nabla_{f^{(k)}}\ell(f_t(x), q)\big)^{\!T}\psi_t^{(k)}(x)\Bigg\Vert_2
\]
satisfies
\[
\chi_t(x,q)\approx \frac{c}{T}\,\mathbb{E}\|p(w_t,x)-q(w_t,x)\|_2,
\]
hence the EL2N-SL score may be defined (up to a constant) as
\[
\mathrm{EL2N\text{-}SL}(x)=\frac{1}{T}\,\mathbb{E}\|p(w_t,x)-q(w_t,x)\|_2.
\]
\end{lemma}

\begin{proof}
Differentiate the temperature-scaled softmax to get
\[
\nabla_{f_k} p_i = \frac{1}{T} p_i(\delta_{ik}-p_k).
\]
Thus, for KL$(q\|p)$ loss, the per-logit gradient is
\[
\nabla_{f^{(k)}}\ell(f_t(x), q) \;=\; \frac{1}{T}~\big(p_k - q_k\big),
\]

Substitute into the GraNd definition $\chi_t(x,q)$,
\begin{align*}
\chi_t(x,q)
&= \mathbb{E}\Bigg\Vert \sum_{k=1}^C \frac{1}{T~} (p_k-q_k)\,\psi_t^{(k)}(x)\Bigg\Vert_2 \\
&= \frac{1}{T}\,\mathbb{E}\Bigg\Vert \sum_{k=1}^C (p_k-q_k)\,\psi_t^{(k)}(x)\Bigg\Vert_2,
\end{align*}

Under Assumption~\ref{asm:orth}, the cross-logit gradient terms vanish and as each $\|\psi_t^{(k)}\|_2\!\approx\!c$, therefore
\begin{align*}
\Big\Vert \sum_k (p_k-q_k) \psi_t^{(k)}\Big\Vert_2^2
&= \sum_k (p_k-q_k)^2 \|\psi_t^{(k)}\|_2^2 \\
&\approx c^2\sum_k (p_k-q_k)^2
= c^2\|p-q\|_2^2.
\end{align*}
Taking square-root and expectation gives the claim:
\[
\chi_t(x,q)= \frac{c}{T}\,\mathbb{E}\|p-q\|_2.
\]
Absorbing the constant $c$ into normalization yields the stated EL2N-SL definition.
\end{proof}

\section{Role of teacher strength in SL+KD saturation}
In order to ascertain the role of teacher strength in the SL+KD saturation results observed in the main paper, we take two additional teachers: (a) a MobileNet-V2 (MNV2) model and (b) a weaker RN18 teacher and use them to obtains soft labels for the the student model (ResNet-18). The results, shown in \cref{fig:slkd_diff_teachers}, demonstrate that our key observation still holds in the SL+KD regime across teachers of different strengths, i.e., a saturation in model performance across subsets of varying data quality and size.

\begin{figure}
    \centering
    \includegraphics[width=\linewidth]{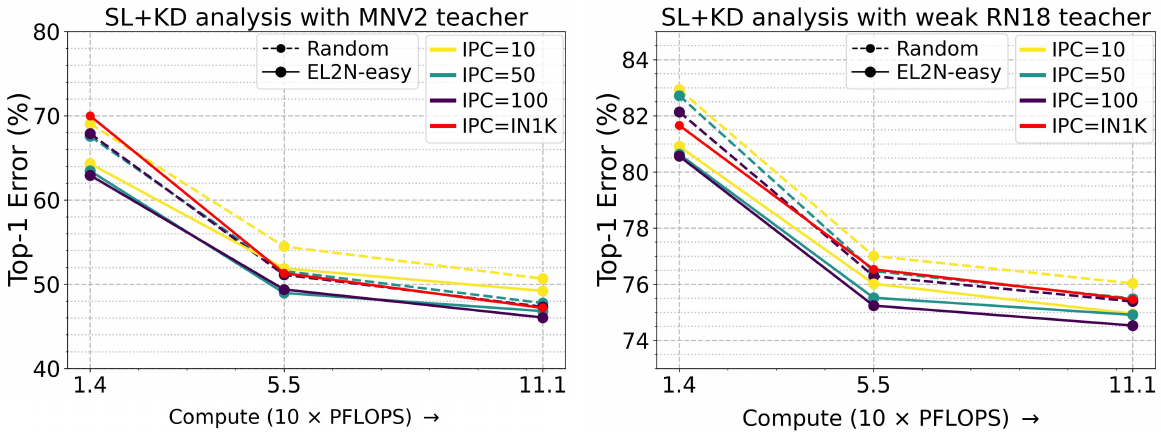}    \caption{\textbf{Role of teacher strength in SL+KD saturation.} We perform scaling analysis in the SL+KD regime with different teachers: (a) MobileNet-V2 (MNV2) teacher, and (b) a \textit{weak} ResNet-18 (RN18) teacher, in order to ascertain the robustness of our observations in the SL+KD regime with varying strength of the teacher models. From the results, we find that, indeed, our observations do hold strong despite varying teacher signals: that \textit{despite varying data quality and size, performance saturation is inevitable given a fixed compute budget}.}
    \label{fig:slkd_diff_teachers}
\end{figure}

\section{Additional analysis for small-scale methods}
\subsection{CIFAR-100 Evaluation}
In this section, we report the results of small-scale DD methods on CIFAR-100 dataset under the fixed HL and SL settings. The three methods considered in this evaluation\footnote{We exclude DATM \cite{guo2024losslessdatasetdistillationdifficultyaligned} from this evaluation due to a significant discrepancy between the performance we observed (specifically, substantially lower accuracy) when running the official DATM code on their provided distilled CIFAR-100 sets and the results reported in their paper. This issue has also been identified by other researchers, for instance, \url{https://github.com/NUS-HPC-AI-Lab/DATM/issues/17}. In contrast, the results for DATM on TinyImageNet are consistent with the reported numbers in their paper. Therefore, we report DATM results only for TinyImageNet.} are DC \cite{zhao2021datasetcondensationgradientmatching}, DM \cite{zhao2022datasetcondensationdistributionmatching} and TM \cite{cazenavette2022datasetdistillationmatchingtraining}. We baseline all the considered small-scale methods against standard coreset techniques (K-Centers \cite{guo2022deepcorecomprehensivelibrarycoreset} and random real subset), which consists of real images and do not involve any costly synthesis process. We adopt the standard setup provided by DCBench \cite{cui2022dcbenchdatasetcondensationbenchmark} for our evaluation. The model architecture is ConvNet-D3, and we compare performance for both IPC 10 and IPC 50. The results are summarized in \cref{tab:small_scale_c100} . The hyperparameters details are provided in \cref{sec:hparam_details} of the supplementary. 

While TM~\cite{cazenavette2022datasetdistillationmatchingtraining} exhibits a \textit{clear advantage over coreset baselines in the HL setting on CIFAR-100}, this advantage substantially diminishes once we transition to the SL regime---\textit{even at low IPC}. For example, although TM exceeds K-Centers by approximately $8$--$13\%$ under HL supervision, it no longer maintains this margin when soft labels are fixed. These results reinforce the pattern observed in \cref{subsec:main_fixed_sl} of the main paper: \textit{subset quality, which is highly influential in the HL setting, plays a far more limited role in the SL regime}, even on a different dataset such as CIFAR-100.

\begin{table}
\centering
\caption{\textbf{Performance comparison of small-scale DD methods with coresets on CIFAR-100 in HL and SL setting}. Model architecture is ConvNet-D3. The substantial performance gap in the HL setting closes when trained with fixed soft labels.}
\setlength{\tabcolsep}{3pt}
\resizebox{0.98\linewidth}{!}{
\label{tab:small_scale_c100}
\begin{tabular}{l|cc|cc}
\toprule
\multicolumn{1}{c|}{\multirow{2}{*}{\textbf{Method}}} & \multicolumn{2}{c|}{\textbf{Hard Label (HL)}} & \multicolumn{2}{c}{\textbf{Fixed Soft Label (SL)}}                                   \\
\cmidrule{2-5}
\multicolumn{1}{c|}{}                                 & \textbf{IPC 10}    & \textbf{IPC 50}    & \multicolumn{1}{c}{\textbf{IPC 10}} & \multicolumn{1}{c}{\textbf{IPC 50}} \\
\midrule
\multicolumn{5}{c}{\textbf{Dataset Distillation}} \\
\midrule
DM  \cite{zhao2022datasetcondensationdistributionmatching}                                                 & 29.23 \std{0.26} & 42.32 \std{0.37}    & 26.13 \std{0.10} & 43.46 \std{0.18}                     \\
DC  \cite{zhao2021datasetcondensationgradientmatching}                                                 & 28.42 \std{0.29} & 30.56 \std{0.56}    & 23.54 \std{0.31} & 33.46 \std{0.38}                    \\
TM \cite{cazenavette2022datasetdistillationmatchingtraining}                                                  & \textbf{38.18 \std{0.42}} & \textbf{46.32 \std{0.26}}    & \textbf{37.60 \std{0.25}} & \textbf{46.26 \std{0.30}}                     \\
\midrule
\multicolumn{5}{c}{\textbf{Coreset Selection}} \\
\midrule
Random Real                                          & 18.64 \std{0.25} & 34.66 \std{0.41}    & 33.43 \std{0.18} & 45.39 \std{0.23}                    \\
K-centers                                               & 25.04 \std{0.30} & 38.64 \std{0.43}    & 34.70 \std{0.27} & \textbf{46.24 \std{0.12}}                    \\
\bottomrule
\end{tabular}}
\end{table}

\subsection{Cross-Arch. Transfer: TinyIN and CIFAR-100} Following the DCBench \cite{cui2022dcbenchdatasetcondensationbenchmark} setup, we evaluate three model architectures -- MLP (411K), ResNet-18 (11M) and ResNet-152 (60M), and report their average transfer accuracy under the sub-heading ``Avg. Transfer" in \cref{tab:small_scale_c100_transfer} and \cref{tab:small_scale_transfer}. When evaluated under the fixed soft-label (SL) setting, coreset methods like K-Centers perform competitively and on-par with state-of-the-art methods like TM and DATM.

\begin{table}
\centering
\caption{\textbf{CIFAR-100 Cross-Architecture Transfer performance comparison of small-scale DD methods with coresets in HL and SL setting}. Model architecture is ConvNet-D3. The substantial performance gap in the HL setting closes when trained with fixed soft labels.}
\setlength{\tabcolsep}{6pt}
\resizebox{0.98\linewidth}{!}{
\label{tab:small_scale_c100_transfer}
\begin{tabular}{l|cc|cc}
\toprule
\multicolumn{1}{c|}{\multirow{2}{*}{\textbf{Method}}} & \multicolumn{2}{c|}{\textbf{Avg. Transfer (HL)}} & \multicolumn{2}{c}{\textbf{Avg. Transfer (SL)}}                                   \\
\cmidrule{2-5}
\multicolumn{1}{c|}{}                                 & \textbf{IPC 10}    & \textbf{IPC 50}    & \multicolumn{1}{c}{\textbf{IPC 10}} & \multicolumn{1}{c}{\textbf{IPC 50}} \\
\midrule
\multicolumn{5}{c}{\textbf{Dataset Distillation}} \\
\midrule
DM  \cite{zhao2022datasetcondensationdistributionmatching}                                                 & 12.44 & 23.59    & 15.69 & 33.51                     \\
DC  \cite{zhao2021datasetcondensationgradientmatching}                                                 & 11.86 & 14.10    & 12.39 & 21.35                    \\
TM \cite{cazenavette2022datasetdistillationmatchingtraining}                                                  & \textbf{15.62} & \textbf{30.29}   & \textbf{23.64} & 36.36                    \\
\midrule
\multicolumn{5}{c}{\textbf{Coreset Selection}} \\
\midrule
Random Real                                          & 10.39 & 21.86    & 20.36 & 35.88                   \\
K-centers                                               & 14.41 & 24.86    & \textbf{22.99} & \textbf{37.03}                   \\
\bottomrule
\end{tabular}}
\end{table}

\section{Additional results on DCS}

\begin{figure*}
    \centering
    \includegraphics[width=0.8\linewidth]{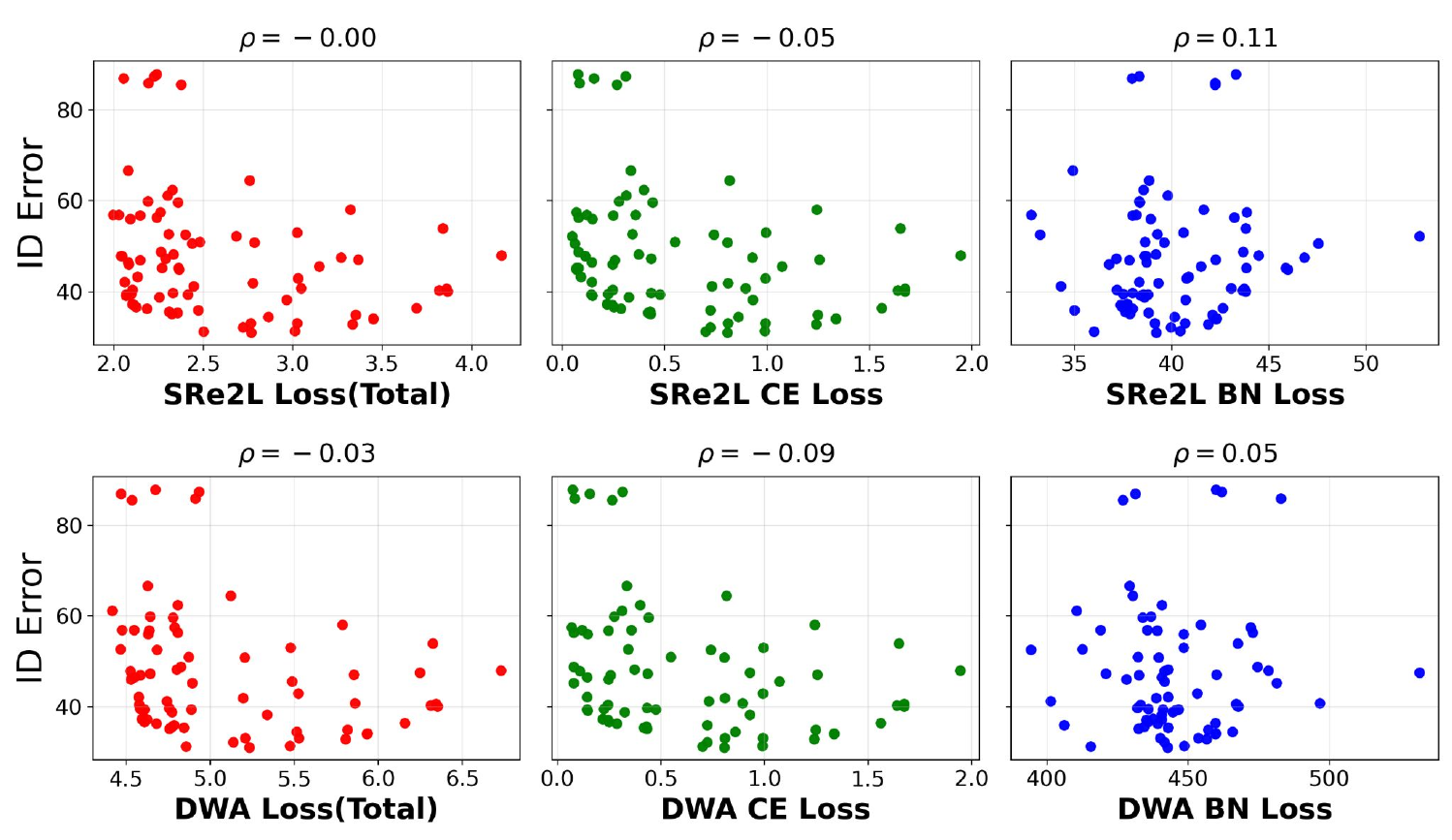}    \caption{\textbf{Correlation analysis of distillation loss objectives on ImageNet-1K}. We compute the proposed DCS score (see Sec-4 of the main paper) for SRe2L and DWA across multiple IPC settings and data subsets (each data point in the plot represents IPC-subset combination, see \cref{subsec:dcs_large_scale} for details and discussion). One can observe a mis-alignment between these distillation objectives and their generalization performance, with either zero or negative Spearman correlation after adjusting for the bias of size of subsets.}
    \label{fig:sre2l_loss_dcs}
\end{figure*}

\begin{figure}
    \centering
    \includegraphics[width=\linewidth]{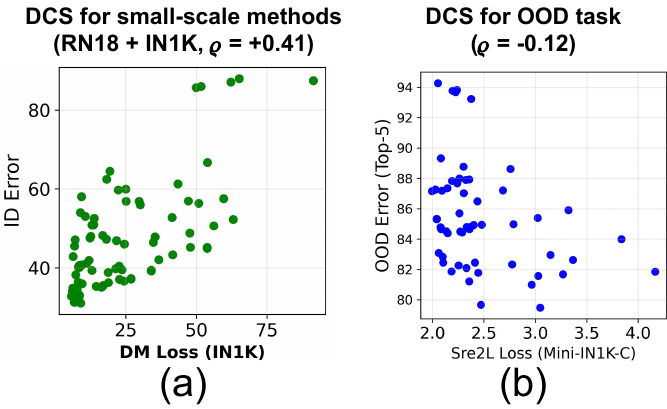}    \caption{\textbf{DCS Additional results}. \textbf{(a) DCS on small-scale method DM}. We use DCS to plot the correlation of DM \cite{zhao2022datasetcondensationdistributionmatching} loss objective with ID generalization error and find better-than-TM but modest correlation of $\rho = 0.41$. \textbf{(b) DCS robustness across tasks.} We also evaluate the robustness of DCS for an OOD task different than the ID task for SRe2L \cite{yin2024squeezerecoverrelabeldataset} on Mini-ImageNet-C dataset, and find that the DCS outputs a similar correlation score across both ID and OOD task ($\rho\sim-0.12$).}
    \label{fig:supple_dcs_dm}
\end{figure}

\begin{figure}
    \centering
    \includegraphics[width=\linewidth]{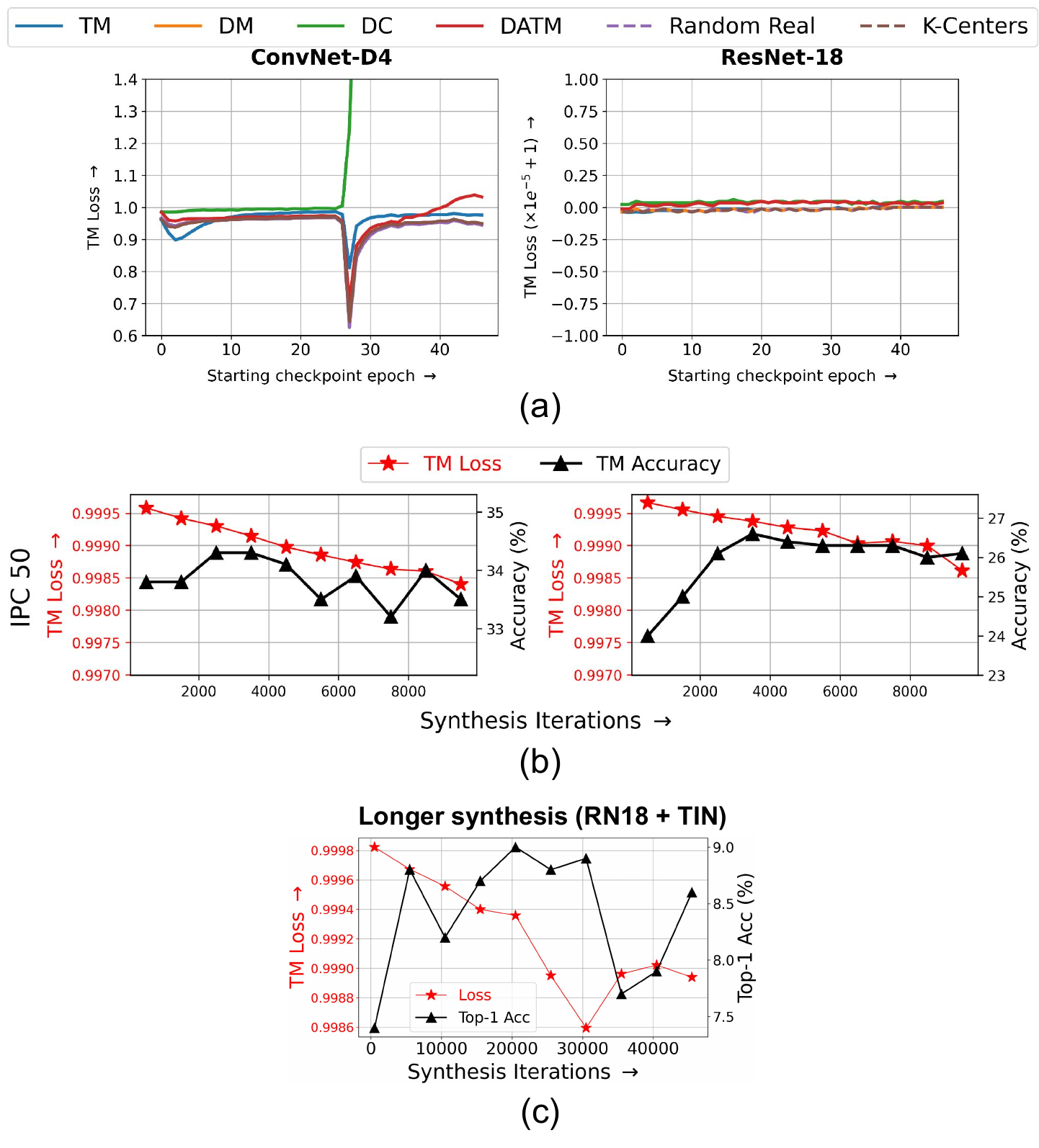}    \caption{\textbf{(a) Analysis of TM Loss objective behavior for different synthesis methods on TinyImageNet.} We calculate the TM Loss for various distilled sets with trajectories starting from different training epochs for both ConvNet-D4 and ResNet-18 model. For deeper and larger architectures like ResNet-18, TM Loss fails to capture any meaningful variation, settling around a constant value of $\sim0.806$ regardless of the method considered. Note that DC loss suddenly shoots up around the dip point (epoch $27-28$) for ConvNet-D4, and hence we clip the loss range to display meaningful variations. \textbf{(b) Training dynamics of DATM synthesis on TinyImageNet.} We track DATM synthesis for ConvNet-D4 (left) and ResNet-18 (right) at IPC 50. Both TM loss (\textcolor{red}{red} curve) and accuracy (\textbf{black} curve) show minimal change -- loss varies only in the 3rd decimal place, and accuracy improves by just $2 - 3\%$ over its initial accuracy. \textbf{(c) Longer Synthesis with DATM loss.} We ascertain if the training dynamics of DATM loss presented in (b) is because of slower convergence. For this, we perform $5 \times$ longer synthesis with the DATM loss than in (b) to check for convergence, and observe no observable trend (increasing Acc or decereasing Loss) in the curves, validating our observation that the TM loss objective indeed does not scale well to larger settings.}
    \label{fig:tm_supple}
\end{figure}

\subsection{Small-scale methods}

\noindent \textbf{TM \cite{cazenavette2022datasetdistillationmatchingtraining}.} In this section, we further examine the behavior of the TM loss objective on TinyImageNet to further validate our observations discussed in Sec. 4.1 of the main paper. For this, we use two model architectures: ConvNet-D4 and ResNet-18. We use the DATM \cite{guo2024losslessdatasetdistillationdifficultyaligned} variant of TM loss for analysis, which sets a lower bound $T^-$ along with an upper bound $T^+$ on the sampling range for the starting checkpoint for matching trajectories. Recall from \cref{subsec:small-scale-related} of the supplementary that the TM loss has two additional hyperparameters $M$ and $N$. For our loss objective analysis experiments, we set $N=80$ and $M=2$.
For both architectures, we visualize two things: (1) The first shows the TM loss variation as a function of the starting expert checkpoint. To obtain the final TM loss value, we average the loss values across $T^-$ and $T^+$ epochs ($T^-=20$ and $T^+=35$ here); (2) we also make a scatter plot of these averaged loss values against the In-Distribution (ID) accuracy of the considered subsets for both IPC $10$ and IPC $50$ (specifically, we consider the six methods discussed in Sec. 3.2 of the main paper).

From the \cref{fig:tm_supple}(a), one can observe that for small models like ConvNet-D4, there is a reasonable variation in the TM Loss for different subsets, which reflect the different generalization performances of those sets. However, for more deeper and larger architectures like ResNet-18, the loss exhibits almost no discernible variation, with its value remaining nearly constant at around $0.805$ for all the considered subsets. Further, as shown in \cref{fig:dcs_tm}(a) of the main paper, the scatter plot of averaged TM Loss for various methods reveal no correlation with in-domain generalization performance of these subsets.

To further ascertain the above observations, we perform the actual synthesis process using the DATM method on TinyImageNet for both ConvNet-D4 and ResNet-18, and track its loss training dynamics. For training, we set N=20 and M=2, based on the maximum available GPU memory for ResNet-18. We include the plots for IPC=50 (\cref{fig:tm_supple}(b)) in the supplementary for completeness. The plots for IPC=10 were included in \cref{subsec:eval_using_dcs} of the main paper.

In these plots, we track both the ``Grand Loss" (TM matching loss across multiple trajectores) as well as the accuracy of the distilled set as one synthesizes for more iterations. Note that despite $10k$ synthesis iterations, one can hardly observe any significant change in both loss and accuracy, with loss value changing in 3rd decimal place and accuracy improving only by $2 - 3\%$ over the base accuracy, which is already relatively high due to initialization of the synthesis set with correctly classified real samples. To confirm that this is not a convergence issue, we also perform a $5 \times$ longer synthesis with the same loss (i.e., $50k$ iterations), and plot its loss and accuracy trend in \cref{fig:tm_supple}(c). The lack of any consistent trend in the accuracy confirms our observation that using TM-based loss objective for synthesis on larger models and datasets is not useful, as it's unable to capture any meaningful changes in the model paramters relevant for downstream generalization. 



\vspace{0.1cm}

\noindent \textbf{DM \cite{zhao2022datasetcondensationdistributionmatching}.} We evaluate the distribution-matching (DM) objective using DCS on the ResNet-18 model with ImageNet-1K (IN1K) dataset, and plot its correlation in \cref{fig:supple_dcs_dm}(a). We find that, although DM has a  much better correlation than TM with downstream generalization, it's still a modest correlation ($\rho = 0.41$) and not very strong. Moreover, we also observe that DM exhibits greater variation in both loss and downstream performance than TM during the synthesis stage.

\begin{table}
\centering
\caption{\textbf{TinyImageNet Cross-Architecture Transfer performance comparison of small-scale DD methods with coresets in HL and SL setting}. Model architecture is ConvNet-D4. The substantial performance gap in the HL setting closes when trained with fixed soft labels.}
\setlength{\tabcolsep}{6pt}
\resizebox{0.98\linewidth}{!}{
\label{tab:small_scale_transfer}
\begin{tabular}{l|cc|cc}
\toprule
\multicolumn{1}{c|}{\multirow{2}{*}{\textbf{Method}}} & \multicolumn{2}{c|}{\textbf{Avg. Transfer (HL)}} & \multicolumn{2}{c}{\textbf{Avg. Transfer (SL)}}                                   \\
\cmidrule{2-5}
\multicolumn{1}{c|}{}                                 & \textbf{IPC 10}    & \textbf{IPC 50}    & \multicolumn{1}{c}{\textbf{IPC 10}} & \multicolumn{1}{c}{\textbf{IPC 50}} \\
\midrule
\multicolumn{5}{c}{\textbf{Dataset Distillation}} \\
\midrule
DM  \cite{zhao2022datasetcondensationdistributionmatching}                                                 & 2.96   & 6.71    & 5.03                     & 19.64                     \\
DC  \cite{zhao2021datasetcondensationgradientmatching}                                                 & 3.60    & 4.41    & 2.49                      & 3.48                     \\
TM \cite{cazenavette2022datasetdistillationmatchingtraining}                                                  & \textbf{5.50}    & \textbf{10.26}    & \textbf{10.27}                     & 20.32                     \\
DATM  \cite{guo2024losslessdatasetdistillationdifficultyaligned}                                               & 5.22              & 10.12               & 8.93                     & \textbf{23.83}                     \\
\midrule
\multicolumn{5}{c}{\textbf{Coreset Selection}} \\
\midrule
Random Real                                          & 2.49     & 6.17    & 9.28                     & 22.91                     \\
K-centers                                               & 3.91    & 7.80    & \textbf{10.51}                     & \textbf{23.53}                   \\
\bottomrule
\end{tabular}}
\end{table}

\subsection{Large-scale methods: SRe2L\cite{yin2024squeezerecoverrelabeldataset} and DWA\cite{du2024diversitydrivensynthesisenhancingdataset}}
\label{subsec:dcs_large_scale}

We employ DCS to evaluate two large-scale distillation objectives: SRe2L \cite{yin2024squeezerecoverrelabeldataset} and DWA \cite{du2024diversitydrivensynthesisenhancingdataset}, which match BatchNorm statistics of a teacher model as a proxy for distribution alignment. A more detailed description of these loss objectives along with their equation is included in Sec-\ref{sec:related} of the supplementary material. We conduct this analysis across seven different IPC values (10 -- 700) with a diverse set of data subsets $\mathcal{S}_j$ coming from four coreset methods (EL2N~\cite{paul2021deep}, EL2N-pareto fractions (see Sec. 3.1 of the main paper), CAL~\cite{margatina2021active} and GraphCut~\cite{pmlr-v132-iyer21a}). For each IPC, we intentionally exclude extremely poor-performing subsets to ensure that the evaluation focuses on whether the loss objectives can meaningfully differentiate among the strongest candidate subsets at that IPC. In addition to these plausible subsets, we also include \textit{adversarial degenerate subsets} by repeating the elements of the subset twice, allowing us to test whether the loss objectives can correctly identify these pathological cases wherein, repetition of the data results in redundancy and a good distillation loss objective should penalize that in favour of diversity and good distribution coverage. All of these subset variants combine to 158 data points per objective for correlation evaluation. 
We choose ResNet-18 and ImageNet-1K dataset as the evaluation benchmark and report Spearman correlation coefficient for each objective after accounting for the confounding effect of subset size, which can otherwise bias the correlation estimates.

The results are shown as scatter plots in \cref{fig:sre2l_loss_dcs}. 
From the figure, it is evident that there is a significant mis / non-alignment between the distillation objectives and generalization performance, with SRe2L and DWA exhibiting \textit{no correlation}, indicating that lower loss values in these objectives have no effect on better downstream generalization.

\subsection{Robustness of DCS}
To evaluate whether DCS remains informative beyond in-distribution settings, we provide additional results on the out-of-distribution (OOD) generalization task. Specifically, we compute DCS for SRe2L on Mini-ImageNet-C and present the corresponding results in \cref{fig:supple_dcs_dm}(b). From the figure, it is evident that there is a significant misalignment between 
the the SRe2L objective and OOD generalization performance as well, which is captured accurately by DCS ($\rho = -0.12$).


 This finding is further corroborated by the poor OOD performance of SRe2L at IPC 50, which achieves only $3.56\%$ accuracy compared to $11.72\%$ for a class-balanced random subset. These results suggest that DCS serves as a reliable indicator of dataset quality across both in-distribution and out-of-distribution evaluation regimes, making it a broadly applicable metric for assessing DD methods.

\section{Additional results on the proposed method}

As shown in the main paper, soft-label regimes (SL and SL+KD) suffer from the undesirable property that \textit{data quality ceases to matter}, limiting the value of quality pruning metrics. Therefore, we focus our additional analyses on the HL regime, where sample quality remains influential, and evaluate both CAD-Prune and CA2D in greater depth.

\subsection{Performance comparison with recent coreset and DD methods}
\label{subsec:comparison_with_latest_methods}
\noindent \textbf{Large-scale methods.} We provide comparisons of our proposed DD method (CA2D) with some of the recently proposed large-scale DD techniques \cite{fadrm_2025} for completeness\footnote{Note that we only consider papers which have been peer-reviewed and published at major AI conference venues.}. Note that these methods may have used potentially different evaluation settings to assess the performance of their distilled sets, including evaluation under a different label regime, or using different training hyper-parameters like no. of epochs for student training. Therefore, we take their publicly available distilled sets and evaluate them under the HL regime in our evaluation setup to maintain a fair comparison. The results for the DD method comparison are shown in \cref{tab:new_sota_compare}, where it is evident that the proposed technique CA2D outpeforms its baselines on IPC 50 \footnote{Note that we only consider IPC 50 for comparison due to unavailability of distilled sets for \citet{fadrm_2025} for any other IPC.}. 

We also compare CA2D and CAD-Prune with some latest state-of-the-art coreset methods such as Dyn-Unc \cite{dyn_unc} and DUAL \cite{pmlr-v267-cho25e}, and report the results in \cref{tab:rebuttal_table_in1k}, where we can observe on-par or better performance of the proposed methods against these baselines.

\vspace{0.1cm}

\noindent \textbf{Small-scale methods.} Small-scale result comparison of the proposed methods with existing popular baselines are provided in \cref{tab:rebuttal_table_ca2d_small_scale} (setting: ConvD3 arch. + CIFAR-100). As expected, due to the very low-resolution nature of the images in CIFAR-100 dataset, the performance is expectedly weaker.

\begin{table}
\centering
\caption{\textbf{Performance comparison of proposed methods with recently proposed DD techniques.} We add comparison against FAD-RM \cite{fadrm_2025} here for completeness, demonstrating that CA2D outperforms it and other DD baselines at IPC 50. See \cref{subsec:comparison_with_latest_methods} for the discussion.}
\setlength{\tabcolsep}{2pt}
\resizebox{\linewidth}{!}{
\label{tab:new_sota_compare}
\begin{tabular}{l|ccccc}
\toprule
\textbf{IN1K + HL} & SRe2L & D4M   & RDED & FAD-RM \cite{fadrm_2025} & CA2D (Ours)  \\
\midrule
\textbf{IPC 50}     & 9.79  & 21.56 & 38.49 & 21.21  &  \textbf{41.56} \\
\bottomrule
\end{tabular}}
\end{table}

\begin{table}[]
\caption{\textbf{Performance comparison of proposed methods with recently proposed coreset techniques.} We add comparison against Dyn-Unc \cite{dyn_unc} and DUAL \cite{pmlr-v267-cho25e} here for completeness, demonstrating that the proposed methods performs on-par or better compared to these baselines. See \cref{subsec:comparison_with_latest_methods} for discussion.}
\setlength{\tabcolsep}{6pt}
\resizebox{0.98\linewidth}{!}{
\label{tab:rebuttal_table_in1k}
\centering
\setlength{\tabcolsep}{4pt}
\begin{tabular}{lccccl}
\toprule
\multicolumn{1}{l}{\textbf{Method}} & \multicolumn{1}{c}{\textbf{IPC 10}} & \multicolumn{1}{c}{\textbf{IPC 50}} & \multicolumn{1}{c}{\textbf{IPC 100}} &  \\
\midrule
Dyn-Unc \cite{dyn_unc}             & 12.15 \std{0.22}                       & 33.03 \std{0.02}                      & 42.01 \std{0.10}                       \\
DUAL \cite{pmlr-v267-cho25e}                  & \textbf{15.65 \std{0.17}}              & 38.28 \std{0.09}                        & 46.22 \std{0.10}                         \\
CAD-Prune (Ours)                    & 12.57 \std{0.03}                        & 40.21 \std{0.15}                       & \textbf{47.40 \std{0.20}}               \\
CA2D (Ours)                         & \textbf{15.25 \std{0.24}}                       & \textbf{41.72 \std{0.40}}              & 46.32 \std{0.10}                       \\
\bottomrule
\vspace{-2.0mm}
\end{tabular}}
\end{table}
\begin{table}[]
\caption{\textbf{Performance comparison of proposed methods with SOTA samll-scale DD and coreset techniques on CIFAR-100 in the HL setting.} Model architecture is ConvNet-D3. Due to the very low-resolution nature of the setup (CIFAR-100), performance of the propoesed techniques is expectedly lower.}
\setlength{\tabcolsep}{10pt}
\resizebox{0.98\linewidth}{!}{
\label{tab:rebuttal_table_ca2d_small_scale}
\centering
\setlength{\tabcolsep}{12pt}
\begin{tabular}{lcc}
\toprule
\multicolumn{1}{l}{\textbf{Method}} & \multirow{1}{*}{\textbf{IPC 10}}                   & \multirow{1}{*}{\textbf{IPC 50}}                   \\
\midrule
TM \cite{cazenavette2022datasetdistillationmatchingtraining}                                                  & \textbf{38.18 \std{0.42}}                                & \textbf{46.32 \std{0.26}}                      \\
CA2D (factor=1)                                         & 20.95 \std{0.15}                                & 34.37 \std{0.11}                                \\
\midrule
Random Real & 18.64 \std{0.25} & 34.66 \std{0.41} \\
K-centers & \textbf{25.04 \std{0.30}} & \textbf{38.64 \std{0.43}} \\
CAD-Prune                                    & 22.95 \std{0.03} & 37.87 \std{0.16} \\
\bottomrule
\end{tabular}}
\end{table}

\subsection{Cross-Architecture Transfer}
We perform subset synthesis / selection using ResNet-18 model architecture and then train a ResNet-50 / ResNet-101 model on the obtained subset. Results are presented in \cref{tab:cross_arch_in1k_method} for both IPC 50 and 100. The proposed compute-aware pruning method CAD-Prune matches the performance of the compute-intensive sliding window method EL2N-Best \cite{paul2021deep,lee2024selmatcheffectivelyscalingdataset} on larger architectures and across IPC values while being more efficient. The proposed DD method CA2D also outperforms RDED \cite{sun2024diversityrealismdistilleddataset} across all settings, demonstrating the benefit of a optimal set selection to construct the distilled set.

\subsection{Convergence Analysis with longer training}
The goal of dataset distillation is to compress a large dataset into a small synthetic one on which a student model trained under a chosen label regime (HL / SL / SL+KD) can still achieve strong downstream performance. We observe a potential artifact of this compression process on the performance of the distilled sets: \textit{Compressing rich feature information into fewer synthetic samples proportionately increases the optimization effort required to extract them in the first place}. We term this as \textit{compression-extraction trade-off} in dataset distillation. 

\begin{figure}
    \centering
    \includegraphics[width=\linewidth]{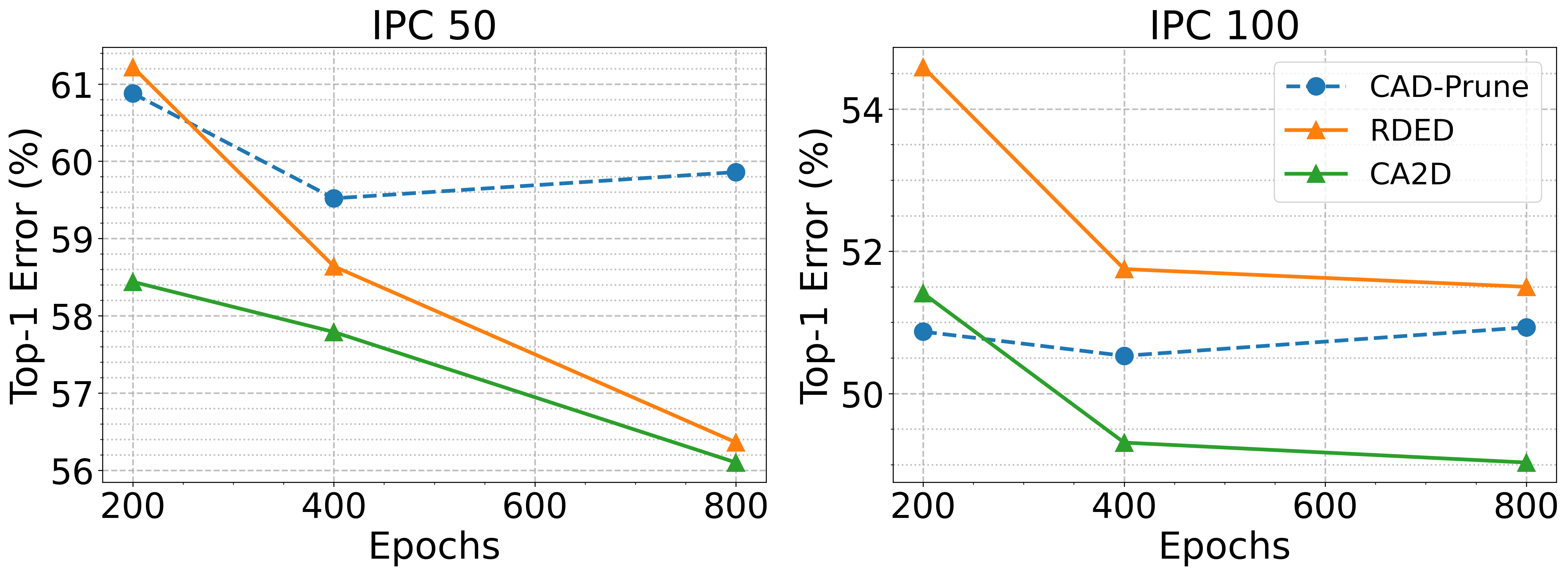}    \caption{\textbf{Convergence analysis of DD methods vs coresets.} We plot downstream student performance (Top-1 Error) as a function of training epochs. One can observe that performance keeps improving for distilled sets (solid line) with longer training, while it saturates for coresets (dashed line), indicating the existence of \textit{compression-extraction trade-off} in training on distilled set.}
    \label{fig:in1k_convergence}
\end{figure}

\begin{table}
\centering
\caption{\textbf{ImageNet-1K Cross-Arch. Transfer performance of proposed methods CAD-Prune and CA2D in HL setting.} Architecture used to perform selection / synthesis is ResNet-18. Student training is for 200 epochs. The proposed compute-aware pruning method CAD-Prune matches the expensive sliding window selection method EL2N-Best while being much more efficient. Similarly, the proposed CA2D built using CAD-Prune outperforms RDED across IPC and architectures. RN50=ResNet-50, RN101=ResNet-101.}
\vspace{0.1cm}
\setlength{\tabcolsep}{8pt}
\resizebox{\linewidth}{!}{
\label{tab:cross_arch_in1k_method}
\begin{tabular}{lccccc}
\toprule
\multicolumn{1}{c}{\multirow{2}{*}{\textbf{Method}}} & \multicolumn{2}{c}{\textbf{IPC 50}}                                            & \multicolumn{2}{c}{\textbf{IPC 100}}                                           \\
\cmidrule{2-5}
\multicolumn{1}{c}{}                                 & \multicolumn{1}{c}{\textbf{RN50}} & \multicolumn{1}{c}{\textbf{RN101}} & \multicolumn{1}{c}{\textbf{RN50}} & \multicolumn{1}{c}{\textbf{RN101}} \\
\midrule
\multicolumn{5}{c}{\textbf{Coreset Selection (224 $\times$ 224)}}  \\
\midrule
Random Real                                          & 32.70                                 & 31.97                                  & 45.20             & 47.01              \\
EL2N-Best                                            & \textbf{40.12}                                 & \textbf{41.18}                                  & 51.24             & \textbf{52.11}              \\
CAD-Prune (Ours)                                     & \textbf{40.37}                                 & \textbf{41.34}                                  & \textbf{52.00}                                    & 49.94                                     \\
\midrule
\multicolumn{5}{c}{\textbf{Dataset Distillation (224 $\times$ 224)}}  \\
\midrule
RDED                                                 & 37.64                & 36.60                 & 47.71                                    & 49.11                                     \\
CA2D (Ours)                                          & \textbf{38.15}                & \textbf{42.16}                 & \textbf{50.41}                                    & \textbf{52.47}     \\
\bottomrule
\end{tabular}}
\end{table}

We demonstrate this in \cref{fig:in1k_convergence}, where we show that, unlike coreset methods (like the proposed CAD-Prune), where performance saturation happens with longer training (indicating convergence in the HL regime), distilled sets like RDED and CA2D continue to improve with extended training, especially for smaller IPC settings, indicating that new information is still being extracted from the distilled samples. This analysis provides two key observations pertinent to dataset distillation research, where there exists an inherent trade-off in compression vs information extraction from a synthesized dataset: 
\begin{itemize}
    \item A certain minimum amount of compute budget is necessary to extract full information from a compressed set, and
    \item This extraction process is atleast $2 \times$ slower in comparison to coresets, where the underlying samples are from the real data distribution with no additional compression.
\end{itemize}


\subsection{Results on Higher IPC settings}
Although evaluation results in dataset distillation is reported at typical IPC values of 10-100, coreset methods report performance values across the IPC range, focusing on the lossless setting which lean towards higher IPC values. To validate the robustness of our proposed compute-aware pruning method CAD-Prune in such IPC ranges, we evaluate it on IPC values of 200 and 500, and report its performance in \cref{tab:abl_high_ipc}. One can observe that the selection mechanism in CAD-Prune clearly maintains its effectiveness in choosing the best scoring subset for the given downstream compute budget by matching the best results obtainable for that IPC value (exhaustive sliding window approach), as shown by the EL2N-Best results.

\subsection{Choice of the scoring checkpoint}
Recall that the main hypothesis behind our proposed CAD-Prune pruning method is that, given a downstream compute budget for training a student model (e.g., 200 epochs on IPC 50), we select scores from a full-dataset training checkpoint whose training duration is compute-aligned with the given downstream budget. Doing so ensures we pick samples that are actually being actively learned when trained for the given budget. To validate that such a compute-aligned checkpoint selection is optimal for a given budget, we perform an ablation, wherein, we select two checkpoints: (1) early checkpoint (matching downstream compute budget) from a longer training, and (2) final model with compute-aligned full-training on the full dataset; and select coresets based on the scores obtained from these checkpoints, and train a downstream model and evaluate it. The results are reported in \cref{tab:abl_scoring_ckpt}, where we can find that compute-aware (full-training) checkpoint performs much better compared to the early chosen checkpoint from a longer training run, and it is on-par / slightly better than the exhaustive sliding-window approach of finding the best score range for a given IPC (EL2N-Best).

\begin{table}
\centering
\caption{\textbf{Performance comparison of proposed pruning method CAD-Prune with EL2N-Best at high IPC settings}. We demonstrate that the proposed compute-aware pruning method CAD-Prune works on-par (or better) w.r.t a compute-heavy sliding window method of ``EL2N-Best'' on IPC values beyond the ranges evaluated in DD literature (and more typical of coreset literature).}
\vspace{0.1cm}
\setlength{\tabcolsep}{18pt}
\resizebox{\linewidth}{!}{
\label{tab:abl_high_ipc}
\begin{tabular}{lll}
\toprule
\multicolumn{1}{c}{\textbf{Method}} & \multicolumn{1}{c}{\textbf{IPC 200}} & \multicolumn{1}{c}{\textbf{IPC 500}} \\
\midrule
EL2N-Best                           & 56.74 \std{0.03}                       & 64.07 \std{0.04}                       \\
CAD-Prune (Ours)                           & \textbf{57.07 \std{0.04}}                       & \textbf{65.46 \std{0.19}}     \\                 
\bottomrule
\end{tabular}}
\end{table}
\begin{table}
\centering
\caption{\textbf{Performance comparison of coresets with different scoring checkpoints for pruning}. Coresets are chosen based on scores from different checkpoints. Depending on a downstream compute budget (e.g., 200 epochs of IPC 50), choosing a compute-aligned training checkpoint (``Compute-aware'') performs much superior to early checkpoint of a longer training.}
\vspace{0.1cm}
\setlength{\tabcolsep}{4pt}
\resizebox{\linewidth}{!}{
\label{tab:abl_scoring_ckpt}
\begin{tabular}{lll}
\toprule
\multicolumn{1}{l}{\textbf{Scoring checkpoint}} & \multicolumn{1}{l}{\textbf{IPC 50}} & \multicolumn{1}{l}{\textbf{IPC 100}} \\
\midrule
Early of full-training                          & 30.95 \std{0.35}                      & 39.94 \std{0.25}                       \\
Compute-aware                       & \textbf{40.21 \std{0.15} \textcolor{ForestGreen}{(+ 9.3)}}                       & \textbf{47.40 \std{0.20} \textcolor{ForestGreen}{(+ 7.5)}}                       \\
\midrule
EL2N-Best & 38.44 \std{0.30} & \textbf{47.76 \std{0.43}} \\
\bottomrule
\end{tabular}}
\end{table}


\section{Image Visualization}
Finally, we provide visualization of synthetic / selected images of various methods considered in the work for qualitative comparison, including our proposed pruning method CAD-Prune and the proposed DD method CA2D. They are shown in \cref{fig:plot_monastery_663} -- \cref{fig:plot_garbage_truck_569} for different classes of ImageNet-1K. Each row corresponds to one coreset / DD method, and we display five images per method for a particular class.

\begin{figure*}
\centering
\includegraphics[width=0.75\linewidth]{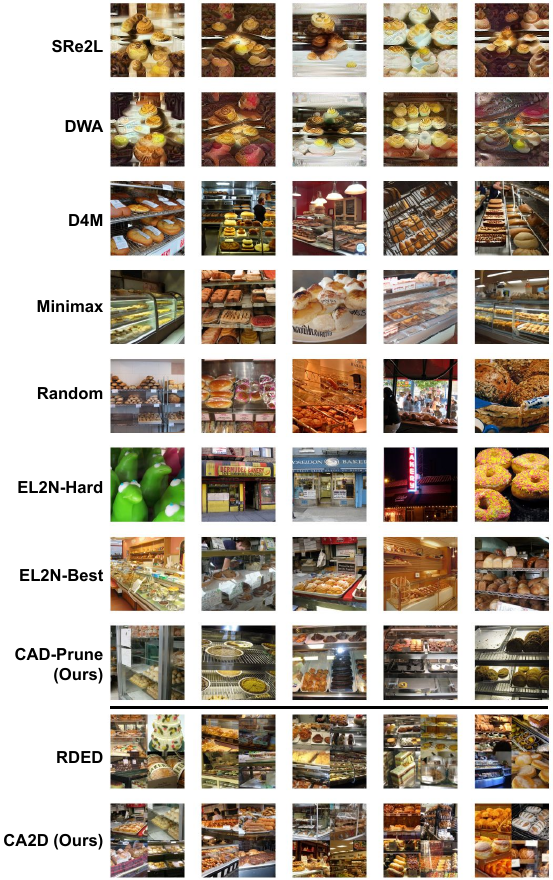}
\caption{\textbf{Class: Bakery} in ImageNet-1K.}
\label{fig:plot_monastery_663}
\end{figure*}

\begin{figure*}
\centering
\includegraphics[width=0.75\linewidth]{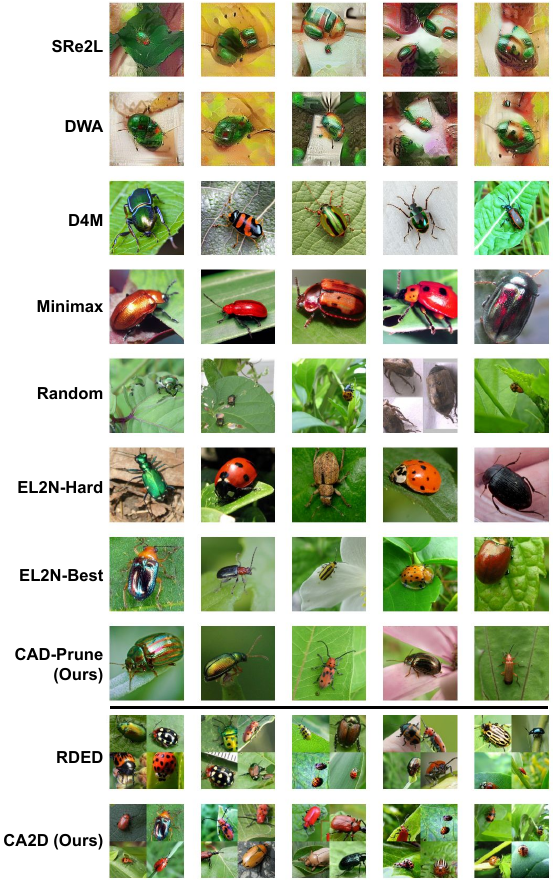}
\caption{\textbf{Class: Leaf Beetle} in ImageNet-1K.}
\label{fig:plot_leaf_beetle_304}
\end{figure*}

\begin{figure*}
\centering
\includegraphics[width=0.75\linewidth]{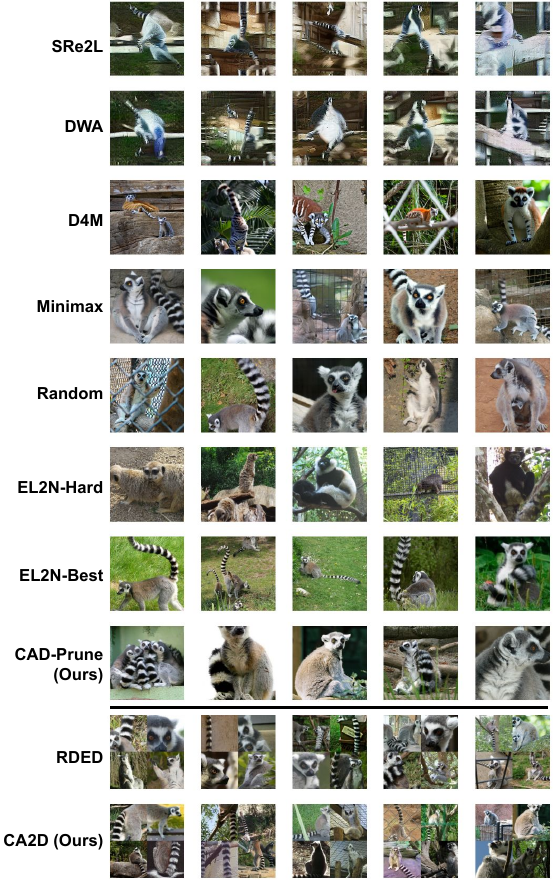}
\caption{\textbf{Class: Madagascar Cat} in ImageNet-1K.}
\label{fig:plot_Madagascar_cat_383}
\end{figure*}

\begin{figure*}
\centering
\includegraphics[width=0.75\linewidth]{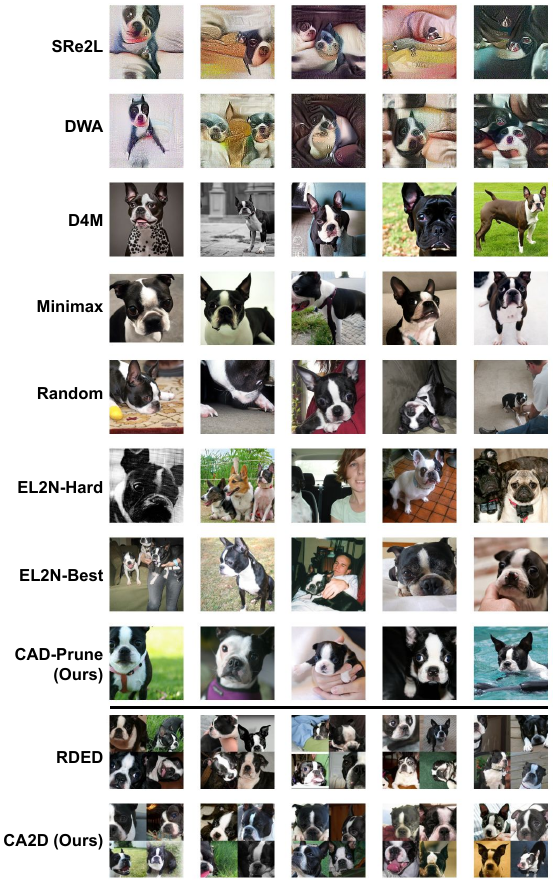}
\caption{\textbf{Class: Boston Bull} in ImageNet-1K.}
\label{fig:plot_Boston_bull_195}
\end{figure*}

\begin{figure*}
\centering
\includegraphics[width=0.75\linewidth]{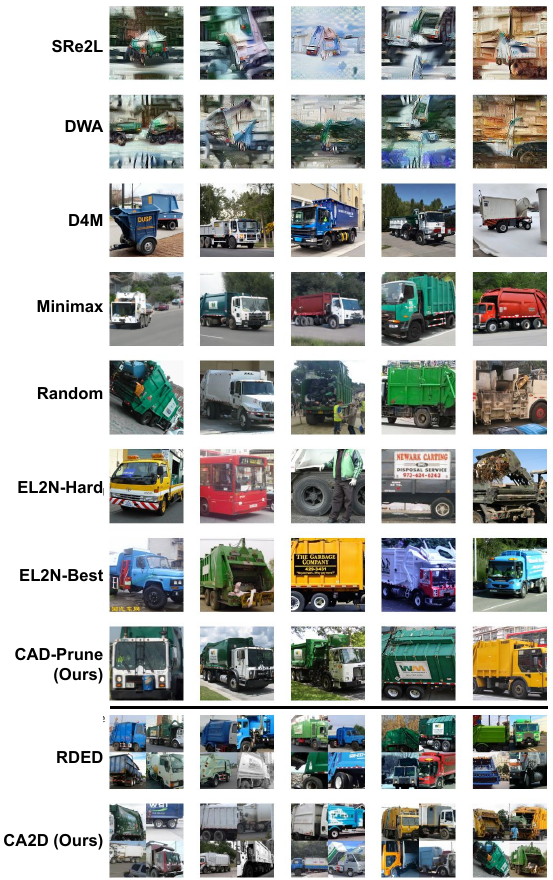}
\caption{\textbf{Class: Garbage Truck} in ImageNet-1K.}
\label{fig:plot_garbage_truck_569}
\end{figure*}

\clearpage


\end{document}